\documentclass[letterpaper]{article} % DO NOT CHANGE THIS
\usepackage{aaai25}  % DO NOT CHANGE THIS
\usepackage{times}  % DO NOT CHANGE THIS
\usepackage{helvet}  % DO NOT CHANGE THIS
\usepackage{courier}  % DO NOT CHANGE THIS
\usepackage[hyphens]{url}  % DO NOT CHANGE THIS
\usepackage{graphicx} % DO NOT CHANGE THIS
\usepackage{natbib}  % DO NOT CHANGE THIS AND DO NOT ADD ANY OPTIONS TO IT
\usepackage{caption} % DO NOT CHANGE THIS AND DO NOT ADD ANY OPTIONS TO IT
\usepackage{algorithm}
\usepackage{algorithmic}

\usepackage{newfloat}
\usepackage{listings}
\DeclareCaptionStyle{ruled}{labelfont=normalfont,labelsep=colon,strut=off} % DO NOT CHANGE THIS
\floatstyle{ruled}
\newfloat{listing}{tb}{lst}{}
\floatname{listing}{Listing}
\usepackage{amsmath,amssymb,amsthm}
\usepackage{multirow}
\usepackage{subfigure}
\usepackage{booktabs}
\newtheorem{assumption}{Assumption}
\newtheorem{corollary}{Corollary}
\newtheorem{theorem}{Theorem}

\usepackage{color}
\newcommand{\pengxin}[1]{\textcolor{black}{#1}}

\title{A New Federated Learning Framework Against Gradient Inversion Attacks}
\author{
    Pengxin Guo\textsuperscript{\rm 1}\equalcontrib, Shuang Zeng\textsuperscript{\rm 2}\equalcontrib, Wenhao Chen\textsuperscript{\rm 1}, Xiaodan Zhang\textsuperscript{\rm 3}, Weihong Ren\textsuperscript{\rm 4}, Yuyin Zhou\textsuperscript{\rm 5}, Liangqiong Qu\textsuperscript{\rm 1}\footnote{Corresponding author.}
}
\affiliations{
    \textsuperscript{\rm 1}School of Computing and Data Science, The University of Hong Kong
    \textsuperscript{\rm 2} Department of Mathematics, The University of Hong Kong
    
    \textsuperscript{\rm 3} College of Computer Science, Beijing University of Technology
    
    \textsuperscript{\rm 4}  School of Mechanical Engineering and Automation, Harbin Institute of Technology, Shenzhen
    
    \textsuperscript{\rm 5} Department of Computer Science and Engineering, UC Santa Cruz
    
    \{guopx, zengsh9\}@connect.hku.hk, wc365@cam.ac.uk, zhangxiaodan@bjut.edu.cn, renweihong@hit.edu.cn, zhouyuyiner@gmail.com, liangqqu@hku.hk
}

\usepackage{bibentry}
\begin{document}

\maketitle

\begin{abstract}
Federated Learning (FL) aims to protect data privacy by enabling clients to collectively train machine learning models without sharing their raw data. However, recent studies demonstrate that information exchanged during FL is subject to Gradient Inversion Attacks (GIA) and, consequently, a variety of privacy-preserving methods have been integrated into FL to thwart such attacks, such as Secure Multi-party Computing (SMC), Homomorphic Encryption (HE), and Differential Privacy (DP). Despite their ability to protect data privacy, these approaches inherently involve substantial privacy-utility trade-offs. By revisiting the key to privacy exposure in FL under GIA, which lies in the frequent sharing of model gradients that contain private data, we take a new perspective by designing a novel privacy preserve FL framework that effectively ``breaks the direct connection'' between the shared parameters and the local private data to defend against GIA. Specifically, we propose a Hypernetwork Federated Learning (HyperFL) framework that utilizes hypernetworks to generate the parameters of the local model and only the hypernetwork parameters are uploaded to the server for aggregation. Theoretical analyses demonstrate the convergence rate of the proposed HyperFL, while extensive experimental results show the privacy-preserving capability and comparable performance of HyperFL. Code is available at https://github.com/Pengxin-Guo/HyperFL.
\end{abstract}

% Uncomment the following to link to your code, datasets, an extended version or similar.
%
% \begin{links}
%     \link{Code}{https://aaai.org/example/code}
%     \link{Datasets}{https://aaai.org/example/datasets}
%     \link{Extended version}{https://aaai.org/example/extended-version}
% \end{links}

\section{Introduction}
\label{sec:intro}

Deep Neural Networks (DNN) have achieved remarkable success on a variety of computer vision tasks, relying on the availability of a large amount of training data \cite{krizhevsky2012imagenet, he2016deep, dosovitskiy2021image, liu2021swin, wang2023learning}. However, in many real-world applications, training data is distributed across different institutions, and data sharing between these entities is often restricted due to privacy and regulatory concerns. To alleviate these concerns, Federated Learning (FL) \cite{mcmahan2017communication, li2020bfederated, qu2022rethinking, zeng2024tackling, guo2024selective, zhang2024flhetbench} has emerged as a promising approach that enables collaborative and decentralized training of AI models across multiple institutions without sharing of personal data externally. %In the FL paradigm, during each training round, clients perform local model updates using the latest copy of the model received from the server. Then, they send the updates (gradients) to the server for aggregation. The server aggregates these updates to construct a global model, which is then broadcasted to all clients for further training. By enabling collaborative training without raw data sharing, FL ensures data privacy protection.

\begin{figure}
  \centering
   \includegraphics[width=0.9\linewidth]{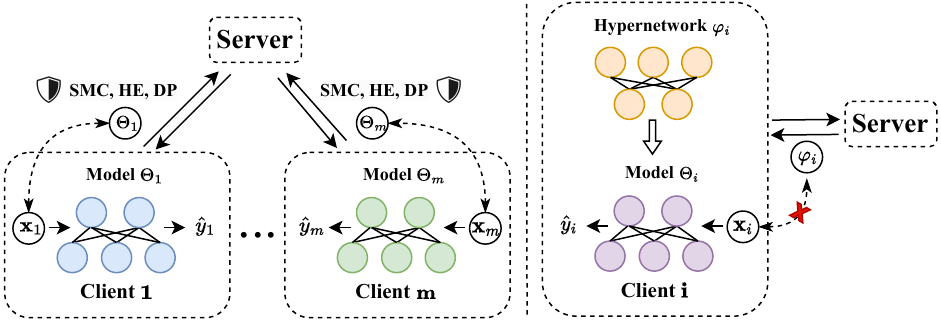}
   % \vskip -0.1in
   \caption{\textit{Left}. Existing methods mainly explore defenses mechanisms on the shared gradients. Such mechanisms, including SMC, HE, and DP, inherently involve substantial privacy-utility trade-offs. \textit{Right}. A novel FL framework that “breaks the direct connection” between the shared parameters and the local private data is proposed to achieve a favorable privacy-utility trade-off.}
\vskip -0.2in
\label{fig:motivation}
\end{figure}

% \begin{wrapfigure}{r}{0.55\textwidth}
%     \centering
%     \vskip -0.1in
%     \includegraphics[width=0.55\textwidth]{motivation-v3.pdf}
%     \vskip -0.1in
%     \caption{\textit{Left}. Existing methods mainly explore defenses mechanisms on the shared gradients. Such mechanisms, including SMC, HE, and DP, inherently involve substantial privacy-utility trade-offs. \textit{Right}. A novel FL framework that “breaks the direct connection” between the shared parameters and the local private data is proposed to achieve a favorable privacy-utility trade-off.}
%     \label{fig:motivation}
%     \vskip -0.2in
% \end{wrapfigure}

Despite the privacy-preserving capability introduced by FL, recent works \cite{geiping2020inverting, huang2021evaluating, hatamizadeh2023gradient} have revealed that FL models are vulnerable to Gradient Inversion Attacks (GIA) \cite{fredrikson2015model, zhu2019deep}. GIA can reconstruct clients' private data from the shared gradients, undermining FL's privacy guarantees. To remedy this issue, various defense mechanisms have been integrated into FL, including Secure Multi-party Computing (SMC) \cite{yao1982protocols, bonawitz2017practical, mugunthan2019smpai, mou2021verifiable, xu2022non}, Homomorphic Encryption (HE) \cite{gentry2009fully, zhang2020batchcrypt, zhang2020privacy, ma2022privacy, park2022privacy} and Differential Privacy (DP) \cite{dwork2006differential, geyer2017differentially, mcmahan2018learning, yu2020salvaging, bietti2022personalization, shen2023share}. These approaches primarily rely on existing defense mechanisms to enhance privacy protection against GIA without altering the FL framework, as illustrated in the left part of Figure \ref{fig:motivation}. However, such mechanisms, including SMC, HE, and DP, inherently involve substantial privacy-utility trade-offs. 
For example, SMC and HE, while providing strong security guarantees through encrypted information exchange, entail high computation and communication costs, making them unsuitable for DNN models with numerous parameters \cite{bonawitz2017practical, zhang2020privacy}. Although DP approaches are easier to implement, they often fall short in providing sufficient model protection while preserving accuracy \cite{geyer2017differentially, mcmahan2018learning}. These trade-offs within current defense algorithms have inspired us to explore alternative privacy-preserving methods that strike a better balance between privacy and utility, leading to the central question of this paper:
%though SMC and HE provide a high-level security guarantee by encrypting exchanged information among clients, its extremely high computation and communication cost make it unsuitable to DNN models that typically consist of numerous parameters \cite{bonawitz2017practical, zhang2020privacy}. While DP can be applied to DNN models to enhance privacy protection, it often introduces significant computation cost and leads to a degradation in model performance \cite{geyer2017differentially, mcmahan2018learning}. 
%The high costs associated with these methods have motivated researchers to explore alternative privacy-preserving techniques that achieve a better balance between privacy and efficiency, which raises the following question this paper aims to address:

\begin{center}
    \textit{\quad Can we design a novel FL framework that offers a favorable privacy-utility trade-off against GIA without relying on existing defense mechanisms?}
    %Can we design a novel FL framework against GIA \quad while without imposing additional computation cost and without causing performance degradation?}
\end{center}

To this end, we revisit the key to privacy exposure in FL under GIA, which lies in the frequent sharing of model gradients that contain private data. 
Most efforts aim at exploring various advanced defenses mechanisms on the shared gradients to enhance privacy preservation in FL, as shown in the left part of Figure \ref{fig:motivation}. In contrast, we take a new perspective striving for designing a novel privacy preserve FL framework that ``breaks the direct connection'' between the shared parameters and the local private data to defend against GIA.
In order to achieve this, we explore the potential of hypernetworks, a class of deep neural networks that generate the weights for another network \cite{david2017hypernetworks}, as a promising solution, as illustrated in the right part of Figure \ref{fig:motivation} and Figure \ref{fig:framework}. 

Specifically, we introduce a novel Hypernetwork Federated Learning (HyperFL) framework that adopts a dual-pronged approach—network decomposition and hypernetwork sharing— to ``break the direct connection'' between the shared parameters and the local private data, as shown in Figure~\ref{fig:framework}. In light of recent findings regarding minimal discrepancies in feature representations and considerable diversity in classifier heads among FL clients \cite{collins2021exploiting, xu2023personalized, shen2023share}, we decompose each local model into a shared feature extractor and a private classifier, enhancing performance in heterogeneous settings and mitigating privacy leakage risks. %This decomposition helps improve performance in heterogeneous settings and increase data privacy protection, as predicting feature extractor outputs is generally more challenging than estimating classifier outputs, due to their higher dimensionality. 
To further strengthen privacy preservation, we employ an auxiliary hypernetwork that generates feature extractor parameters based on private client embeddings. Instead of directly sharing the feature extractor, only the hypernetwork parameters are uploaded to the server for aggregation, while classifiers and embeddings are trained locally. This auxiliary hypernetwork sharing strategy ``breaks the direct connection" between shared parameters and local private data, maintaining privacy while enabling inter-client interaction and information exchange. 

Remarkably, the design of HyperFL is flexible and scalable, catering to a diverse range of FL demands through its various configurations. We present two major configurations of HyperFL: (1) Main Configuration HyperFL, suitable for simple tasks with small networks, learns the entire feature extractor parameters directly (see Figure \ref{fig:framework}); (2) HyperFL for Large Pre-trained Models (denoted as HyperFL-LPM) targets for complex tasks by using pre-trained models as fixed feature extractors and generating trainable adapter parameters via a hypernetwork \cite{houlsby2019parameter} (see Figure \ref{fig:framework_v2}).
Both theoretical analysis and extensive experimental results demonstrate that HyperFL effectively preserves privacy under GIA while achieving comparable results and maintaining a similar convergence rate to FedAvg \cite{mcmahan2017communication}. We hope that the proposed HyperFL framework can encourage the research community to consider the importance of developing new enhanced privacy preservation FL frameworks, as an alternative to current research efforts on defense mechanisms front.

We summarize our contributions as follows:
\begin{itemize}
    \item To defend against GIA, we take a new perspective by designing a novel privacy preserve FL framework that effectively ``breaks the direct connection'' between the shared parameters and the local private data and propose the HyperFL framework.
    \item We present two major configurations of HyperFL: (1) Main Configuration HyperFL, suitable for simple tasks with small networks, learns the entire feature extractor parameters directly; (2) HyperFL-LPM targets for complex tasks by using pre-trained models as fixed feature extractors and generating trainable adapter parameters via a hypernetwork.
    \item Both theoretical analysis and extensive experimental results demonstrate that HyperFL effectively preserves privacy under GIA while achieving comparable results and maintaining a similar convergence rate to FedAvg.
\end{itemize}

%The paper is organized as follows. Section \ref{sec:related_work} introduces some related works. Section \ref{sec:method} describes our framework in detail. Section \ref{sec:theo_ana} establishes some theoretical analyses to demonstrate the privacy protection ability and convergence rate of our method. Section \ref{sec:exp} shows experimentally that HyperFL achieves better performance and faster convergence rate compared with DP-based FL methods on several datasets. Section \ref{sec:conclusion} concludes this paper.

\section{Related Work}
\label{sec:related_work}

\begin{figure*}
  \centering
   \includegraphics[width=0.8\linewidth]{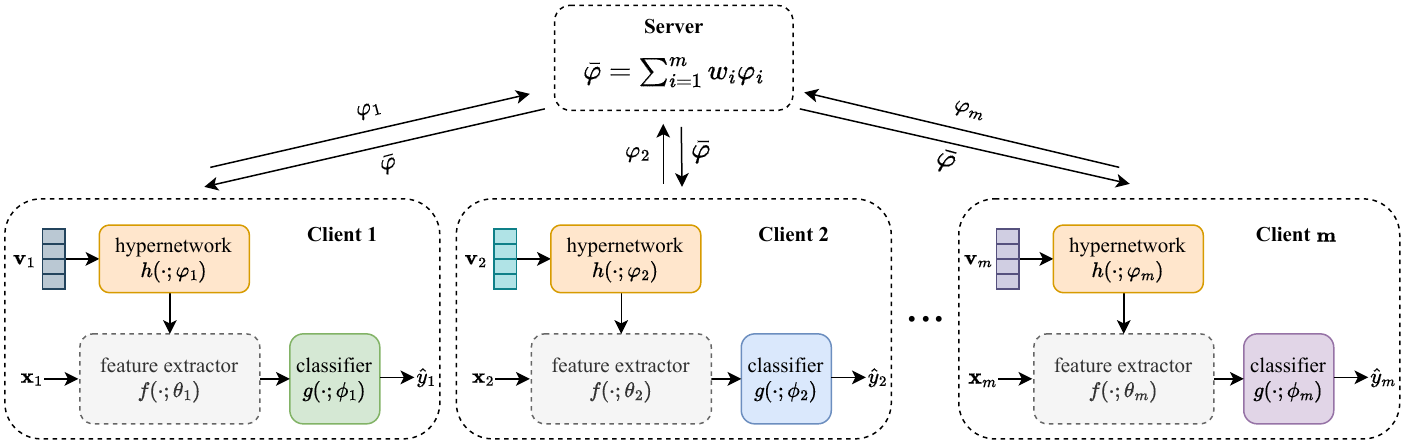}
   % \vskip -0.1in
   \caption{The proposed HyperFL framework. HyperFL decouples each client’s network into the former feature extractor $f(; \theta_i)$ and the latter classifier head $g(;{\phi _i})$. An auxiliary hypernetwork $h(;{\varphi _i})$ is introduced to generate local clients’ feature extractor $f(; \theta_i)$  using the client’s private embedding vector $\mathbf{v}_i$, i.e.,  ${\theta _i} = h({{\bf{v}}_i};{\varphi _i})$. These generated parameters are then used to extract features from the input ${x}_i$, which are subsequently fed into the classifier to obtain the output $\hat{y}_i$, expressed as $\hat{y}_i = g( f({x}_i; \theta_i); \phi_i)$. Throughout the FL training, \textbf{only} the hypernetwork  $\varphi_i$ is shared, while all other components are kept private, thus effectively mitigating potential privacy leakage concerns.}
   %During the local training of each client, a hypernetwork is utilized to generate the parameters of the feature extractor by taking the client embedding $\mathbf{v}_i$ as input, i.e., $\theta_i = h(\mathbf{v}_i; \varphi_i)$. These generated parameters are then used to extract features from the input ${x}_i$, which are subsequently fed into the classifier to obtain the output $\hat{y}_i$, expressed as $\hat{y}_i = g( f({x}_i; \theta_i); \phi_i)$. Then the stochastic gradient decent algorithm can be adopted to update the client embedding $\mathbf{v}_i$, the hypernetwork parameters $\varphi_i$, and the classifier parameters $\phi_i$. After local training, \textbf{only} the hypernetwork parameters $\varphi_i$ are uploaded to the server for aggregation. Subsequently, the aggregated hypernetwrok parameters $\bar{\varphi}$ are broadcasted to each client for further rounds of local training.}
   
\label{fig:framework}
\vskip -0.2in
\end{figure*}

\paragraph{Gradient Inversion Attacks.} 
% Gradient Inversion Attacks (GIA) \cite{fredrikson2015model, zhu2019deep} is a class of adversarial attacks that exploit the gradients of a machine learning model to infer sensitive information about the training data by leveraging the fact that gradients contain information about the relationship between the input and the model's output. The basic idea behind GIA is to intentionally modify the input data in a way that maximizes the magnitude of the gradients with respect to the sensitive information of interest. By iteratively adjusting the input data based on the gradients, an attacker can gradually approximate the sensitive information, such as private attributes or training data samples, that the model was trained on \cite{phong2017privacy, zhu2019deep, geiping2020inverting, yin2021see, luo2022effective, geng2023improved, kariyappa2023cocktail}. GIA can pose a significant threat to privacy in scenarios where the model is used in sensitive applications or when the model's training data contains sensitive information. These attacks highlight the need for robust privacy protection mechanisms to mitigate the risk of information leakage through gradients.
Gradient Inversion Attacks (GIA) \cite{fredrikson2015model, zhu2019deep} are adversarial attacks that exploit a machine learning model's gradients to infer sensitive information about the training data. It iteratively adjusts the input data based on gradients to approximate private attributes or training samples \cite{phong2017privacy, zhu2019deep, geiping2020inverting, yin2021see, luo2022effective, geng2023improved, kariyappa2023cocktail}. 
% These attacks pose a privacy threat, especially in sensitive applications or when the model's training data contains sensitive information, emphasizing the need for robust privacy protection mechanisms to prevent information leakage through gradients.

\paragraph{Privacy Protection in Federated Learning.} 
% Although the local data are not exposed in FL, the exchanged model gradients may still leak sensitive information about the data that can be leveraged by GIA to recover them \cite{geiping2020inverting, huang2021evaluating, li2022e2egi, hatamizadeh2023gradient, kariyappa2023cocktail}. 
To protect the data privacy in FL, additional defense methods have been integrated into FL, such as
% and can be categorized into three classes: 
Secure Multi-party Computing (SMC) \cite{yao1982protocols} based methods \cite{bonawitz2017practical, mugunthan2019smpai, mou2021verifiable, xu2022non}, Homomorphic Encryption (HE) \cite{gentry2009fully} based methods \cite{zhang2020batchcrypt, zhang2020privacy, ma2022privacy, park2022privacy} and Differential Privacy (DP) \cite{dwork2006differential} based methods \cite{geyer2017differentially, mcmahan2018learning, yu2020salvaging, bietti2022personalization, shen2023share}. 
Apart from these defense methods with theoretical guarantees, there are other empirical yet effective defense strategies, such as gradient pruning/masking \cite{zhu2019deep, huang2021evaluating, li2022auditing}, noise addition \cite{zhu2019deep, wei2020framework, huang2021evaluating, li2022auditing}, Soteria \cite{sun2020provable}, PRECODE \cite{scheliga2022precode}, and FedKL \cite{ren2023gradient}. However, these methods always suffer from privacy-utility trade-off problems, as illustrated in their papers. In contrast to these approaches, with the help of hypernetworks \cite{david2017hypernetworks}, this work proposes a novel FL framework that effectively “breaks the direct connection” between the shared parameters and the local private data to defend against GIA while achieving a favorable privacy-utility trade-off.

\paragraph{Hypernetworks in Federated Learning.} Hypernetworks \cite{david2017hypernetworks} are deep neural networks that generate the weights for another network, known as the target network, based on varying inputs to the hypernetwork. Recently, there have been some works that incorporate hypernetworks into FL for learning personalized models \cite{shamsian2021personalized, carey2022robust, li2023fedtp, tashakori2023semipfl, lin2023federated}. All of these methods adopt the similar idea that a central hypernetwork model are trained on the server to generate a set of models, one model for each client, which aims to generate personalized model for each client. 
Since the hypernetwork and client embeddings are trained on the server, which makes the server possessing all the information about the local models, enabling the server to recover the original inputs by GIA (see Table \ref{tab:attack} and Figure \ref{fig:attack} in Appendix).
In contrast to existing approaches, this work presents a Hypernetwork Federated Learning (HyperFL) framework, which prioritizes data privacy preservation over personalized model generation through the utilization of hypernetworks.
% Since the hypernetwork and client embeddings are trained on the server, which makes the server possessing all the information about the local models, enabling the server to recover the original inputs by GIA (see Figure \ref{fig:attack}). To defend against GIA, we employ two key strategies in HyperFL that distinguish it from previous works: (1) training the client embedding locally and not sharing it with the server, and (2) decomposing each local model into a feature extractor (with parameters are generated by the hypernetwork), and a private classifier trained locally. Both modifications help improve privacy preservation (see Section \ref{sec:any_privacy}). Additionally, the design of our HyperFL is flexible, catering to a diverse range of FL demands through various configurations, such as HyperFL for simple tasks, and HyperFL-LPM for complex tasks.

\noindent A more detailed discussion on related work is provided in Appendix.

\section{Method}
\label{sec:method}

In this section, we first formalize the FL problem, then we present our HyperFL framework.
% We primarily focus on introducing the Main Configuration HyperFL framework and briefly discuss the HyperFL-LPM framework since they share the same training procedure.

\subsection{Problem Formulation}

In FL, suppose there are $m$ clients and a central server, where all clients communicate to the server to collaboratively train their models without sharing raw private data. Each client $i$ is equipped with its own data distribution $P_{XY}^{(i)}$ on $\mathcal{X} \times \mathcal{Y}$, where $\mathcal{X}$ is the input space and $\mathcal{Y}$ is the label space with $K$ categories in total. 
% We assume that $P_{XY}^{(i)}$ and $P_{XY}^{(j)}$ are different for any pair of client $i$ and $j$, which is usually the case in FL.
Let $\ell: \mathcal{X} \times \mathcal{Y} \rightarrow \mathbb{R}_+$ denotes the loss function given local model $\Theta_i$ and data point sampled from $P_{XY}^{(i)}$, then the underlying optimization goal of FL can be formalized as follows
{\small
\begin{equation} \label{eq:fl_obj_1}
    \arg \min _{\Theta} \frac{1}{m} \sum_{i=1}^m \mathbb{E}_{({x}, y) \sim P_{X Y}^{(i)}}\left[\ell\left(\Theta_i ; {x}, y\right)\right],
\end{equation}
}where $\Theta = \{\Theta_1, \Theta_2, \ldots, \Theta_m\}$ denotes the collection of all local models. In vanilla FL, all clients share the same parameters, i.e., $\Theta_1 = \Theta_2 = \cdots = \Theta_m$. In contrast, personalized FL allows for variation in the parameters across clients, enabling $\Theta_i$ to be different for each client.

Since the true underlying data distribution of each client is inaccessible, the common approach to achieving the objective (\ref{eq:fl_obj_1}) is through Empirical Risk Minimization (ERM). That is, assume each client has access to $n_i$ i.i.d. data points sampled from $P_{XY}^{(i)}$ denoted by $\mathcal{D}_i=\left\{\left({x}_i^l, y_i^l\right)\right\}_{l=1}^{n_i}$, whose corresponding empirical distribution is $\hat{P}_{XY}^{(i)}$, and we assume the empirical marginal distribution $\hat{P}_{XY}^{(i)}$ is identical to the true $P_{XY}^{(i)}$. Then the training objective is
{\small
\begin{equation} \label{eq:fl_obj_2}
    \arg \min _{\Theta} \frac{1}{m} \sum_{i=1}^m \mathcal{L}_i(\Theta_i),
\end{equation}
}where $\mathcal{L}_i(\Theta_i) = \frac{1}{n_i} \sum_{l=1}^{n_i} \ell(\Theta_i; {x}_i^l, y_i^l)$ is the local average loss over personal training data, e.g., empirical risk.

% \subsection{Our Framework}
\subsection{Main Configuration HyperFL}

In the Main Configuration HyperFL framework, which is shown in Figure \ref{fig:framework}, each client $i$ has a classification network parameterized by $\Theta_i = \{\theta_i, \phi_i\}$ consists of a feature extractor $f: \mathcal{X} \rightarrow \mathbb{R}^d$ parameterized by $\theta_i$ , and a classifier $g: \mathbb{R}^d \rightarrow \mathbb{R}^K$ parameterized by $\phi_i$, where $d$ is the feature dimension and $K$ is the number of classes. Additionally, each client $i$ has a private client embedding $\mathbf{v}_i$ and a hypernetwork $h$ parameterized by $\varphi_i$, which is responsible for generating the parameters of the feature extractor $f$, i.e., $\theta_i = h(\mathbf{v}_i; \varphi_i)$. In this way, the hypernetwork can generate personalized feature extractor parameters for each client by taking the meaningful client embedding as input. The client embeddings can be trainable vectors or fixed vectors, depending on whether suitable client representations are known in advance. In this work, we adopt trainable vectors. Then, the objective (\ref{eq:fl_obj_2}) can be reformulated as
{\small
\begin{equation} \label{eq:fl_obj_ours}
    \arg \min _{\varphi, \phi, \mathbf{v}} \frac{1}{m} \sum_{i=1}^m \mathcal{L}_i(h(\mathbf{v}_i; \varphi_i), \phi_i),
\end{equation}
}where $\varphi = \{\varphi_1, \varphi_2, \ldots, \varphi_m\}$, $\phi = \{\phi_1, \phi_2, \ldots, \phi_m\}$, $\mathbf{v} = \{\mathbf{v}_1, \mathbf{v}_2, \ldots, \mathbf{v}_m\}$. Note that the feature extractor parameters are generated by the hypernetwork and not trainable, whereas the client embedding, the hypernetwork, and the classifier parameters are trainable.

Then, in order to “breaks the direct connection” between the shared parameters and the local private data to defend against GIA while maintaining competitive performance, each client only uploads the hypernetwork parameters to the server for aggregation while keeping the classifier and the private client embedding trained locally. As illustrated in Figure \ref{fig:framework}, in each FL communication round, each client $i$ uploads its hypernetwork parameters $\varphi_i$ to the server once the local training is completed while keeps the classifier parameters $\phi_i$ and client embedding $\mathbf{v}_i$ local to strengthen privacy protection. Then, the server aggregate these $\varphi_i$ to obtain the global $\bar{\varphi}$. Next, clients download $\bar{\varphi}$ to replace their corresponding local hypernetworks and start the next training iteration. This framework provides a natural way for sharing information across clients while maintaining the privacy of each client, by sharing the hypernetwork parameters. 
% That is, it enables inter-client interaction and information exchange while maintaining data privacy, which contributes to learning the relationships between clients.  
We will elaborate on this workflow in the following.

\paragraph{Local Training Procedure.} For local model training at each round, we first replace the local hypernetwork parameters $\varphi_i$ by the received aggregated hypernetwork parameter $\bar{\varphi}$. Then, we perform stochastic gradient decent steps to iteratively train the model parameters as follows:
\begin{itemize}
    \item \textbf{Step 1: Fix $\varphi_i$ and $\mathbf{v}_i$, update $\phi_i$}. Train the classifier parameters $\phi_i$ by gradient descent for one epoch:
    {\small
    \begin{equation} \label{eq:update_classifier}
    \phi_i \leftarrow \phi_i - \eta_g \nabla_{\phi_i} \ell\left(h(\mathbf{v_i}; \varphi_i), \phi_i ; \xi_i\right),
    \end{equation}
    }where $\xi_i$ denotes the mini-batch of data, $\eta_g$ is the learning rate for updating the classifier parameters.
    
    \item \textbf{Step 2: Fix new $\phi_i$, update $\varphi_i$ and $\mathbf{v}_i$}. After getting new classifier, we proceed to update the hypernetwork parameters $\varphi_i$ and client embedding $\mathbf{v}_i$ for multiple epochs:
    {\small
    \begin{align} \label{eq:update_hypernetwork}
    &\varphi_i \leftarrow \varphi_i - \eta_h \nabla_{\varphi_i} \ell\left(h(\mathbf{v_i}; \varphi_i), \phi_i ; \xi_i\right) \nonumber \\
    &\mathbf{v}_i \leftarrow \mathbf{v}_i - \eta_v \nabla_{\mathbf{v}_i} \ell\left(h(\mathbf{v_i}; \varphi_i), \phi_i ; \xi_i\right),
    \end{align}
    }where $\eta_h$ is the learning rate for updating the hypernetwork parameters and $\eta_v$ is the learning rate for updating the client embedding.
\end{itemize}

\paragraph{Global Aggregation.} Similar to common FL algorithms, the server performs weighted averaging of the hypernetwork parameters as
{\small
\begin{equation} \label{eq:server_agg}
\bar{\varphi} = \sum_{i=1}^m w_i \varphi_i,
\end{equation}
}where $w_i$ is the aggregation weight for client $i$, usually determined by the local data size, i.e., $w_i = \frac{n_i}{\sum_{i=1}^m n_i}$.

\subsection{HyperFL-LPM}

\begin{figure}[h]
  \centering
   \includegraphics[width=0.8\linewidth]{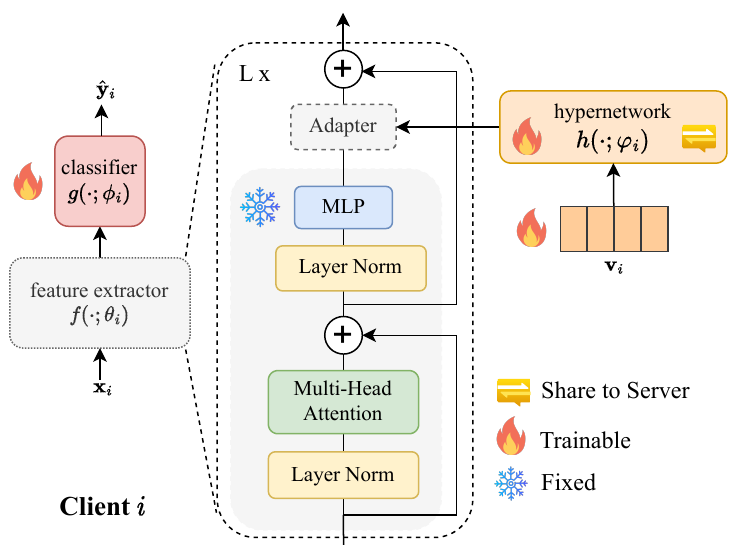}
   % \vskip -0.1in
   \caption{The proposed HyperFL-LPM framework within each client. In this framework, the weights of the pre-trained model are fixed, while only the classifier, hypernetwork, and client embedding are trainable. Note that $\theta$ here represents the parameters of the adapters.}
\label{fig:framework_v2}
\vskip -0.1in
\end{figure}

% \begin{wrapfigure}{r}{0.45\textwidth}
% \vskip -0.2in
%     \centering
%     \includegraphics[width=0.45\textwidth]{HyperFLv2.pdf}
%     \vskip -0.1in
%     \caption{The proposed HyperFL-LPM framework within each client. In this framework, the weights of the pre-trained model are fixed, while only the classifier, hypernetwork, and client embedding are trainable. Note that $\theta$ here represents the parameters of the adapters.}
%     \label{fig:framework_v2}
%     \vskip -0.1in
% \end{wrapfigure}

When confronted with a large feature extractor, using a hypernetwork to generate the parameters can be challenging. However, in scenarios where the feature extractor is significantly sizable, there are often numerous pre-trained models that are readily available \cite{he2016deep, vaswani2017attention, dosovitskiy2021image, liu2021swin, he2022masked}. Consequently, we can leverage these pre-trained models and employ parameter-effective fine-tuning techniques to adapt and fine-tune them \cite{houlsby2019parameter, hu2022lora, jia2022visual, guo2024selective}. To this end, we extend our HyperFL framework to address this situation and propose HyperFL for Large Pre-trained Models (HyperFL-LPM). 
The difference between HyperFL-LPM and Main Configuration HyperFL is the model adopted withen each client. In the HyperFL-LPM famework, instead of using the hypernetwork to generate the entire feature extractor parameters, we employ it to generate the adapter parameters to fine-tune the pre-trained models. Taking the Transformer-based pre-trained models as an example, the framework within each client is shown in Figure \ref{fig:framework_v2}.
By adopting this approach, our framework can utilize large pre-trained models as fixed feature extractors to handle complex tasks.

\section{Analysis and Insights}
\label{sec:theo_ana}

\subsection{Privacy Protection Analysis} \label{sec:any_privacy}

In this section, we present a comprehensive privacy analysis of our HyperFL framework. We consider the most common and widely adopted setting, where the server is an \textit{honest-but-curious} adversary, which obeys the training protocol but attempts to obtain the private data of clients according to model weights and updates \cite{liu2022privacy,li2023temporal}. 
  %We consider the setting that the server is an \textit{honest-but-curious} adversary, which obeys the training protocol but attempts to obtain the private data of clients according to model weights and updates \cite{li2023temporal}. Before conducting analyses of the privacy preservation abilities of the HyperFL framework, it is essential to provide a detailed introduction to GIA.

% \vspace{-0.4cm}
\paragraph{Background: Gradient Inversion Attacks.} %Previous studies \cite{phong2017privacy, zhu2019deep, geiping2020inverting, yin2021see, luo2022effective, geng2023improved, kariyappa2023cocktail} have demonstrated the feasibility of recovering input data from gradients in image classification tasks. 
Given a neural network with parameters $\Theta$ and the gradients $\nabla_{\Theta} \mathcal{L}_{\Theta} (x^*, y^*)$ computed with a private data batch $(x^*, y^*)$, GIA tries to recover $x$, an approximation of $x^*$ as:
{\small
\begin{equation} \label{eq:attack}
    \arg \min _x \mathcal{L}_{\text {grad }}\left(x ; \Theta, \nabla_\Theta \mathcal{L}_\Theta\left(x^*, y^*\right)\right) + \alpha \mathcal{R}_{\text {aux }}(x),
\end{equation}
}where $\mathcal{L}_{\text {grad }}\left(x ; \Theta, \nabla_\Theta \mathcal{L}_\Theta\left(x^*, y^*\right)\right)$ is a gradient loss term used to enforce matching the gradients of recovered batch $x$ with the provided gradients $\mathcal{L}_\Theta(x^*, y^*)$, $\mathcal{R}_{\text {aux }}(x)$ is a regularization term utilized to regularize the recovered image based on image priors, and $\alpha$ is a regularization coefficient. The differences between previous works lie in the choice of gradient loss terms and regularization terms \cite{zhu2019deep, geiping2020inverting, yin2021see, luo2022effective, geng2023improved, kariyappa2023cocktail}. For example, Zhu et al. \cite{zhu2019deep} use $\ell_2$-distance as $\mathcal{L}_{\text {grad }}$ but do not use a regularization term $\mathcal{R}_{\text {aux }}$. Geiping et al. \cite{geiping2020inverting} adopt cosine similarity as $\mathcal{L}_{\text {grad }}$ and the total variance as $\mathcal{R}_{\text {aux }}$. Luo et al. \cite{luo2022effective} utilize cosine similarity as $\mathcal{L}_{\text {grad }}$ and apply two types of regularization within $\mathcal{R}_{\text {aux }}$: one gradient regularization for fully connected layer and another total variation regularization for convolution features.  Geng et al. \cite{geng2023improved} adopt $\ell_2$-distance as $\mathcal{L}_{\text {grad }}$ and divide $\mathcal{R}_{\text {aux }}$ into three terms, total variation on the input $x$, clip and scale operation on the input $x$.

% \vspace{-0.4cm}
% \paragraph{Why HyperFL can Defend against GIA?} \label{sec:why_work} 
\paragraph{Analysis on Proposed Leakage Defense.} \label{sec:why_work} 
Unlike previous FL models where the entire model parameters are uploaded to the server for aggregation, the HyperFL framework only requires each client to upload the hypernetwork parameters to the server. Therefore, the objective function of an attacker tries to recover $x$ from the HyperFL framework should be changed as:
{\small
\begin{equation} \label{eq:att_ours}
    \arg \min _{x} \mathcal{L}_{\text {grad }}\left(x; \varphi, \nabla_\varphi \mathcal{L}_{\varphi, \phi}\left(x^*, y^*, \mathbf{v}^*\right)\right) + \alpha \mathcal{R}_{\text {aux }}(x),
\end{equation}
}where $\varphi$ denotes the hypernetwork parameters, $\phi$ is the parameters of the classifier, $\mathbf{v}^*$ is the client embedding. Note that there is a major difference between objective ($\ref{eq:att_ours}$) and objective ($\ref{eq:attack}$). In objective ($\ref{eq:attack}$), the server can obtain gradients of the entire model, while in objective ($\ref{eq:att_ours}$), it can only access the gradients of the hypernetwork. 

Then, given that $x$ is only exposed to the feature extractor, obtaining the information of the feature extractor is essential to recover $x$. However, in HyperFL, the parameters of the feature extractor are obtained by inputting the client embedding into the hypernetwork. Therefore, it is necessary to first recover the client embedding. To achieve this pipeline, the objective (\ref{eq:att_ours}) can be reformulated as a bi-level optimization problem:
{\small
\begin{equation} \label{eq:att_ours_bilevel}
\begin{aligned}
&\arg \min _{x} \mathcal{L}_{\text {grad }}\left(x, \hat{\mathbf{v}}; \varphi, \Delta_\theta\right) + \alpha \mathcal{R}_{\text {aux }}(x)  \\
\mathrm{s.t.}~~&\hat{\mathbf{v}} = \arg \min _{\mathbf{v}} \mathcal{L}_{\text {grad }}\left(x, \mathbf{v}; \varphi, \nabla_\varphi \mathcal{L}_{\varphi, \phi}\left(x^*, y^*, \mathbf{v}^*\right)\right),
\end{aligned}
\end{equation}
}where $\mathbf{v}$ is an approximation of $\mathbf{v}^*$, $\theta$ is the parameters of the feature extractor that generated by the hypernertwork $h$, i.e., $\theta = h(\mathbf{v}; \varphi)$, $\Delta_\theta = \theta_t - \theta_{t-1}$ serves as an approximation for the gradient of the feature extractor $\theta$
% \footnote{Due to the unavailability of the gradients of the feature extractor at the server, this approximation serves as an alternative method to compute the gradients.} 
\cite{zhang2019lookahead}. However, a challenge arises when solving the lower-level subproblem in objective (\ref{eq:att_ours_bilevel}). According to the chain rule, $\nabla_\varphi \mathcal{L} (x, y^*, \mathbf{v}) = \frac{\mathcal{L}(x, y^*, \mathbf{v})}{\partial \phi}\frac{\partial \phi}{\partial \varphi}$. To compute the gradients of the hypernetwork, it is necessary to calculate the gradients of the classifier first \footnote{We assume the label information can be obtained \cite{geiping2020inverting, zhao2020idlg, yin2021see, ma2023instance}.}. However, since the classifier is trained locally and not shared with the server, it is not feasible to compute these gradients. As a result, the gradients of the hypernetwork cannot be determined, making it challenging to recover the client embedding.

One may question whether we can eliminate the need for the classifier in the process of recovering the client embedding. Drawing inspiration from the GIA procedure \cite{zhu2019deep, li2023temporal}, we can replace the label information that needs to be optimized with the output of the hypernetwork in this context. By employing this approach, it's able to bypass the requirement of the classifier. Then, the lower-level subproblem in objective (\ref{eq:att_ours_bilevel}) will be reformulated as 
{\small
\begin{equation} \label{eq:recover_v}
    \arg \min _{\mathbf{v}, \theta} \mathcal{L}_{\text {grad }}\left(\theta, \mathbf{v}; \varphi, \nabla_\varphi \mathcal{L}_{\varphi, \phi}\left(x^*, y^*, \mathbf{v}^*\right)\right),
\end{equation}
}where $\theta$ is the outputs of the hypernetwork. 
% However, this problem is hard to optimize due to the high dimensionality of the unknown variables \cite{candes2006robust, benning2018modern}. Thus, it's still challenging to recover the client embedding.
However, previous works have shown that simulating the optimization of both the input and output is challenging \cite{zhu2019deep, zhao2020idlg, ma2023instance}. Therefore, researchers propose to first identify the output and then optimize the input 
\cite{zhao2020idlg, geiping2020inverting, zhu2021r, yin2021see, ma2023instance, wang2024towards}. 
Specifically, they can identify the label $y^*$ based on the relationship between the known gradient and label information, as $y^*$ is typically low-dimensional (i.e., a simple one-hot vector) \cite{zhao2020idlg, yin2021see, ma2023instance}. However, the output of our hypernetwork (i.e., $\theta$) is high-dimensional and complex. This complexity makes it challenging to identify the ground-truth output $\theta^*$ using the known gradient information, and consequently, recovering the embedding becomes difficult. Even if we attempt to optimize both the input and output simultaneously, solving Eq. (\ref{eq:recover_v}) remains challenging due to the large search space \cite{zhu2019deep, dang2021revealing, huang2021evaluating, kariyappa2023cocktail}. Thus, it's challenging to recover the client embedding. Furthermore, even if the client embedding can be recovered (albeit with significant error), the input $x$ is still difficult to recover due to the same problems (i.e., inability to infer the output first and a large search space) encountered when solving the upper-level subproblem in objective (\ref{eq:att_ours_bilevel}).

In summary, the HyperFL framework effectively safeguards data privacy, as the combination of the hypernetwork, locally trained classifier, and private client embedding renders the recovery of $x$ using GIA unattainable, which is also demonstrated in our experiments.
%thanks to the hypernetwork and the locally trained classifier and client embedding, it's impossible to recover $x$ by GIA from the HyperFL framework, which demonetsrates the effectiveness of the HyperFL framework to better preserve data privacy.

\subsection{Convergence Analysis} \label{sec:convergence_ana}
To facilitate the convergence analysis of HyperFL, we make the assumptions commonly encountered
in literature \cite{li2019convergence} to characterize the smooth and non-convex optimization landscape.

\begin{assumption} 
\label{assumption1}
$\mathcal{L}_1, \cdots, \mathcal{L}_m$\text{ are all L-smooth: for all} $(\phi_j,$
\noindent $\varphi_j,\mathbf{v}_j)$ \text{and} $(\phi_k,\varphi_k,\mathbf{v}_k)$, 
$\mathcal{L}_i(h(\mathbf{v}_k,\varphi_k),\phi_k) \leq \notag \mathcal{L}_i(h(\mathbf{v}_j,$
\noindent $\varphi_j),\phi_j) + ((\phi_k,\varphi_k,\mathbf{v}_k)-(\phi_j,\varphi_j,\mathbf{v}_j))\nabla \mathcal{L}_i(h(\mathbf{v}_j,\varphi_j),$\\\noindent $\phi_j)+\frac{L}{2}\|(\phi_k,\varphi_k,\mathbf{v}_k)-(\phi_j,\varphi_j,\mathbf{v}_j)\|_2^2
$.
\end{assumption}
\begin{assumption}
\label{assumption2}
Let $\xi_i^t$ be sampled from the $i$-th client's local data uniformly at random at $t$-th training step. The variance of stochastic gradients in each client for each variable is bounded: 
$\mathbb{E}\left\|\nabla_\phi L_i\left(h\left(\mathbf{v}_i^t, \varphi_i^t\right), \phi_i^t, \xi_i^t\right)-\nabla_\phi L_i\left(h\left(\mathbf{v}_i^t, \varphi_i^t\right), \phi_i^t\right)\right\|^2 \leq $
\noindent $\sigma_i^2,\mathbb{E}\left\|\nabla_\varphi L_i\left(h\left(\mathbf{v}_i^t, \varphi_i^t\right), \phi_i^{t+1}, \xi_i^t\right)-\nabla_\varphi L_i\left(h\left(\mathbf{v}_i^t, \varphi_i^t\right), \phi\right.\right.$
\noindent $\left.\left._i^{t+1}\right)\right\|^2 \leq \sigma_i^2$,
$\mathbb{E}\left\|\nabla_\mathbf{v} L_i\left(h\left(\mathbf{v}_i^t, \varphi_i^t\right), \phi_i^{t+1}, \xi_i^t\right)-\nabla_\mathbf{v} L_i\left(h\right.\right.$
\noindent $\left.\left.\left(\mathbf{v}_i^t, \varphi_i^t\right), \phi_i^{t+1}\right)\right\|^2 \leq \sigma_i^2$ 
\text{ for } $i=1, \cdots, m$.
\end{assumption}

\begin{assumption}
\label{assumption3}
The expected squared norm of stochastic gradients is uniformly bounded, i.e., 
$\mathbb{E}\left\|\nabla_\phi L_i\left(h\left(\mathbf{v}_i^t, \varphi_i^t\right),\right.\right.$

\noindent $\left.\left. \phi_i^t, \xi_i^t\right)\right\|^2 \leq G^2, \mathbb{E}\left\|\nabla_\varphi L_i\left(h\left(\mathbf{v}_i^t, \varphi_i^t\right), \phi_i^{t+1}, \xi_i^t\right)\right\|^2\leq G^2$, 
\noindent $\mathbb{E}\left\|\nabla_\mathbf{v} L_i\left(h\left(\mathbf{v}_i^t, \varphi_i^t\right), \phi_i^{t+1}, \xi_i^t\right)\right\|^2 \leq G^2\text{~for all~} i=1,\cdots,$
\noindent $m \text{ and } t=0, \cdots, T-1$. Here $T$ denotes the total number of every client's training steps.
\end{assumption}

Then we present the convergence rate for HyperFL.
\begin{theorem}
\label{theorem1}
Let Assumptions \ref{assumption1}, \ref{assumption2} and \ref{assumption3} hold and $L$, $M$, $\sigma_i$, $G$ be defined therein. Denote $\eta_{\min }=\min \left\{\eta_g, \frac{1}{2} \eta_h, \eta_v\right\}$ and $E$ as the number of local training iterations between two communication rounds. Then we have
{\small
\begin{equation}
\frac{1}{m T} \sum_{i=1}^m \sum_{t=1}^T \mathbb{E}\left[\left\|\nabla\mathcal{L}_i^t\right\|^2\right] \leq 2 \sqrt{\frac{L M G^2 D}{2 T}},
\end{equation}
}where $\mathcal{L}_i^0 - \mathcal{L}_i^* \leq D, \forall i$, and
$
\eta_g^2+\eta_v^2+(E-1) \eta_h^2+\frac{E-1}{L} \eta_h \leq M \eta_{min}^2.
$
\end{theorem} 
According to Theorem \ref{theorem1}, we can obtain an $O(\frac{1}{\sqrt{T}})$ convergence rate towards the stationary solution under smooth and non-convex conditions. This convergence rate is comparable to that of FedAvg in the non-convex scenario \cite{yu2019parallel}. Furthermore, we can expedite the convergence for Polyak-Lojasiewicz (PL) functions \cite{karimi2016linear}, which are commonly encountered in non-convex optimization scenarios.
\begin{assumption}
\label{pl_assumption}
% Polyak-Lojasiewicz (PL) Assumption: 
A function $f$ is $\mu$-PL function if for some $\mu > 0$, it satisfies
{\small
$$
\|\nabla f(x)\|^2 \geq 2\mu\left(f(x) - \inf_{x^{\prime}}f(x^{\prime})\right), \quad \forall x.
$$
}We assume all $\mathcal{L}_1, \cdots, \mathcal{L}_m$ are $\mu$-PL functions, and simply denote $\inf_{x^{\prime}}f(x^{\prime})$ by $f^*$.
\end{assumption}

\begin{corollary}
\label{corollary1}
With assumptions as well as $\eta_{min}$, $L$, $M$ and $D$ defined in Theorem \ref{theorem1} and extra Assumption \ref{pl_assumption}, we have
{\small
\begin{equation}
\begin{aligned}
& \mathbb{E}\left[\frac{1}{m} \sum_{i=1}^m \mathcal{L}_i\left(h\left(\mathbf{v}_i^t ; \varphi_i^t\right), \phi_i^t\right)\right]-\mathcal{L}^* \\ \leq &\left(1-2\eta_{\min } \mu\right)^{t+1} D+\eta_{\min } \frac{L M G^2}{4 \mu}.
\end{aligned}
\end{equation}
}If we set $\eta_{\min } \leq \frac{\mu \epsilon}{L M G^2}$, after $O\left(\frac{1}{\epsilon} \log \left(\frac{1}{\epsilon}\right)\right)$ steps, we have that
{\small
\begin{equation}
\mathbb{E}\left[\frac{1}{m} \sum_{i=1}^m \mathcal{L}_i\left(h\left(\mathbf{v}_i^t ; \varphi_i^t\right), \phi_i^t\right)\right]-\mathcal{L}^* \leq \epsilon.
\end{equation}
}
\end{corollary}
When employing PL functions, the convergence rate of HyperFL is faster than that achieved solely through smoothness assumptions. 
% The detailed proofs for Theorem \ref{theorem1} and Corollary \ref{corollary1} are provided in Appendix.
% \ref{convergence_proof}.

\section{Experiments}
\label{sec:exp}

% In this section, we compare the performance of the proposed HyperFL against several FL baselines to demonstrate HyperFL's ability to enhance data privacy while maintaining or even improving model performance. %not only enhances data protection but also avoids performance degradation even achieves better results.

\subsection{Experimental Setup} \label{sec:exp_setup}

\textbf{Datasets.} For the Main Configuration HyperFL, we evaluate our method on four widely-used image classification datasets: (1) EMNIST \cite{cohen2017emnist}; (2) Fashion-MNIST \cite{xiao2017fashion}; (3) CIFAR-10 \cite{krizhevsky2009learning}; and (4) CINIC-10 \cite{darlow2018cinic}.
% Similar to \cite{karimireddy2020scaffold, zhang2021personalized, xu2023personalized}, we create a non-IID data distribution by ensuring all clients have the same data size, in which $s\%$ of data ($20\%$ by default) are uniformly sampled from all classes and the remaining $(100 - s)\%$ from a set of dominant classes for each client. Following \cite{xu2023personalized}, we evenly divide all clients into multiple groups, with each group having the same dominant classes. Specifically, for  the 10-category Fashion-MNIST, CIFAR-10 and CINIC-10 datasets, we divide clients into 5 groups. Each group is assigned three consecutive classes as the dominant class set, starting from class 0, 2, 4, 6, and 8 for the respective groups. For EMNIST dataset, we divide clients into 3 groups, with each group assigned the dominant set of digits, uppercase letters, and lowercase letters, respectively.
For the HyperFL-LPM, we evaluate our method on the EMNIST \cite{cohen2017emnist} and CIFAR-10 \cite{krizhevsky2009learning} datasets.

% \vspace{-0.1cm}
\textbf{Model Architectures.} For the Main Configuration HyperFL, simlar to \cite{xu2023personalized}, we adopt two different CNN target models for EMNIST/Fashion-MNIST and CIFAR-10/CINIC-10, respectively.
% The first CNN target model is built with two convolutional layers. The first CNN target model is built with two convolutional layers (16 and 32 channels) followed by max pooling layers, two fully-connected layers (128 and 10 units), and a softmax output layer, using LeakyReLU activation functions \cite{xu2015empirical}. The second CNN model is similar to the first one but adds one more 64-channel convolution layer. The hypernetwork is a fully-connected neural network with one hidden layer, multiple linear heads per target weight tensor, and a 64-dimensional learnable client embedding vector. %The first layer has 16 channels, while the second layer has 32 channels. Each convolutional layer is followed by a max pooling layer. The model also includes two fully-connected layers with 128 and 10 units, respectively, before the softmax output layer. For activation functions, the LeakyReLU function \cite{xu2015empirical} is utilized. The second target CNN model is similar to the first one but has one more convolution layer with 64 channels. The hypernetwork is a simple fully-connected neural network, with one hidden layers and multiple linear heads per target weight tensor and the client embedding is a learnable vector with dimension equals 64.
For the HyperFL-LPM, we adopt the ViT \cite{dosovitskiy2021image} and ResNet \cite{he2016deep} pre-trained on the ImageNet dataset \cite{deng2009imagenet} as the feature extractor. 
% The adapter within each transformer block consists of a down-projection layer, ReLU activation functions \cite{nair2010rectified}, and a up-projection layer. 
The hypernetworks of HyperFL and HyperFL-LPM both are a fully-connected neural network with one hidden layer, multiple linear heads per target weight tensor. The client embeddings are learnable vectors with dimension equals 64.

\textbf{Compared Methods.}  For the Main Configuration HyperFL, we compare the proposed method with the following approaches: (1) Local-only; (2) FedAvg \cite{mcmahan2017communication}; (3) pFedHN \cite{shamsian2021personalized}; and some DP-based FL methods, including (4) DP-FedAvg \cite{mcmahan2018learning}; (5) PPSGD \cite{bietti2022personalization}; and (6) CENTAUR \cite{shen2023share}. 
% However, all these compared DP-based FL methods (i.e., DP-FedAvg, PPSGD, and CENTAUR) focus on user-level DP setting \cite{mcmahan2018learning}, which cannot guarantee protection against \textit{honest-but-curious} server attacks as they upload original gradients to the server. Therefore, we adapt these methods to fit the ISRL-DP setting \cite{lowy2023private}, where users trust their own client but not the server or other clients, thereby defending against \textit{honest-but-curious} server attacks.
For the HyperFL-LPM, we compare our method with (1) Local-only with fixed feature extractor; (2) Local-only with adapter fine-tuning; (3) FedAvg with fixed feature extractor; and (4) FedAvg with adapter fine-tuning.

\textbf{Training Settings.} 
% For the Main Configuration HyperFL, mini-batch SGD \cite{ruder2016overview} is adopted as the local optimizer for all approaches. Similar to \cite{xu2023personalized}, we set the step sizes $\eta_h$ and $\eta_v$ for local training to 0.01 for EMNIST/Fashion-MNIST and 0.02 for CIFAR-10/CINIC-10. The setp size $\eta_g$ is set to 0.1 for all the datasets. The weight decay is set to 5e-4 and the momentum is set to 0.5. The batch size is fixed to B = 50 for all datasets except EMNIST (B = 100). The client embedding dimension is set to 64. The number of local training epochs is set to 5 for all FL approaches and the number of global communication rounds is set to 200 for all datasets. Furthermore, following \cite{xu2023personalized}, we conduct experiments on two setups, where the number of clients is 20 and 100, respectively. For the latter, we apply random client selection with sampling rate 0.3 along with full participation in the last round. The training data size per client is set to 600 for all datasets except EMNIST, where the size is 1000. For the DP-based FL methods, the DP budget $\epsilon$ is set to 4 and the Gaussian noise $\sigma$ is $1e-5$ to satisfy the ($\epsilon$, $\sigma$) privacy guarantee. Average test accuracy of all local models is reported for performance evaluation. 
% For the HyperFL-LPM, differently from HyperFL, batch size 16 is adopted for all datasets. Furthermore, the step sizes $\eta_h$ and $\eta_v$ for local training is 0.02 for EMNIST and 0.1 for CIFAR-10.
We employ the mini-batch SGD \cite{ruder2016overview}  as a local optimizer for all approaches, and the number of local training epochs is set to 5. The number of global communication rounds is set to 200 for all datasets. Average test accuracy of all local models is reported for performance evaluation.  

% For privacy evaluation, we set all the unknown variable in HyperFL are learnable and optimized simultaneously, and adopt IG \cite{geiping2020inverting} to recover the input images. More details about experimental setup are provided in Appendix. 

For privacy evaluation, 
% we set all the unknown variable in HyperFL are learnable and optimized simultaneously,
we adopt the widely used IG \cite{geiping2020inverting}, state-of-the-art ROG \cite{yue2023gradient}, and a tailored attack method for our defense framework to recover the input images. More details about experimental setup are provided in Appendix. 

\subsection{Experimental Results}
\label{sec:exp_result}
%\subsubsection{Experimental Results for Privacy-Utility}

\paragraph{Performance Evaluation.}

As demonstrated in Table \ref{tab:main_result}, the performance of all the compared DP-based FL methods is inferior to FedAvg and Local-only. 
%The main experimental results are shown in Table \ref{tab:main_result}. From this table we can see that the performance of all the DP-based FL methods is inferior to Local-only and FedAvg. 
This is due to the incorporation of DP mechanisms, which adversely affect model usability and result in decreased performance. In contrast, our proposed HyperFL consistently surpasses these methods across various datasets, demonstrating its outstanding utility.  %which demonstrates HyperFL achieves the best privacy utility trade-off. 
Notably, HyperFL excels in both situations where Local-only outperforms (i.e., Fashion-MNIST and CINIC-10) and where FedAvg prevails (i.e., EMNIST and CIFAR-10).
%Specifically, HyperFL consistently achieves the best performance, regardless of whether it is on datasets where the Local-only method performs better (i.e., Fashion-MNIST and CINIC-10) or on datasets where the FedAvg method performs better (i.e., EMNIST and CIFAR-10). 
This further highlights HyperFL's adaptability, excelling in both centralized FL scenarios and cases requiring personalization. Furthermore, the learned client embeddings, which are meaningful, can be found in the Appendix.
%  due to page limits. %The learned client embeddings for generating personalized feature extractors, along with the personalized classifier, contribute to this flexibility. As a result, our method is capable of maintaining robust performance across diverse scenarios and dataset types. 
%This further confirms the effectiveness of the learned client embeddings, which can be used to generate personalized feature extractors. As a result, the model can achieve good performance across various types of datasets.
Although pFedHN outperforms our method in two scenarios, it exhibits poor defense capability against GIA, as illustrated in Table \ref{tab:attack}.

\begin{table}[h]
% \begin{wraptable}{r}{0.55\textwidth}
% \vskip -0.2in
\centering
{\fontsize{24}{28}\selectfont
\resizebox{\linewidth}{!}{
\begin{tabular}{lcccccccc}
\toprule 
 \multirow{2}{*}{Method} & \multicolumn{2}{c}{EMNIST} & \multicolumn{2}{c}{Fashion-MNIST} & \multicolumn{2}{c}{CIFAR-10} & \multicolumn{2}{c}{CINIC-10} \\
 \cmidrule(lr){2-3} \cmidrule(lr){4-5} \cmidrule(lr){6-7} \cmidrule(lr){8-9}
 & 20 clients & 100 clients & 20 clients & 100 clients & 20 clients & 100 clients & 20 clients & 100 clients \\

\midrule
Local-only & 73.41 & 75.68 & 85.93 & 87.01 & 65.47 & 66.11 & 63.60 & 64.84 \\
FedAvg & 72.77 & 78.87 & 85.67 & 88.11 & 70.02 & 76.24 & 57.00 & 59.11 \\
pFedHN & \textbf{80.86} & 77.37 & 87.64 & 89.80 & 70.18 & \textbf{80.07} & 63.88 & 70.36 \\
\midrule
DP-FedAvg & 35.12 & 45.73 & 59.88 & 68.29 & 29.12 & 32.03 & 27.30 & 29.94 \\

CENTAUR & 68.82	& 67.24	& 83.07	& 79.77 & 50.85 & 51.86& 48.82 & 51.01 \\

PPSGD & 71.16 & 71.18 & 84.47 &  82.94 & 52.17 & 53.92 & 49.98 & 52.91 \\

\midrule
HyperFL & {76.29} & \textbf{80.22} & \textbf{88.28} & \textbf{90.41} & \textbf{73.03} & {78.73} & \textbf{66.74} & \textbf{72.21} \\

\bottomrule
\end{tabular}
}
}
\vskip -0.1in
\caption{The comparison of final test accuracy (\%) of different methods on various datasets. We apply full participation for FL system with 20 clients, and apply client sampling with rate 0.3 for FL system with 100 clients.}
\label{tab:main_result}
\vskip -0.15in
\end{table}
% \end{wraptable}

The performance of HyperFL-LPM compared with Local-only and FedAvg is shown in Table \ref{tab:result_2}. From this table, we can see that HyperFL-LPM can achieve comparable performance to baseline adapter fine-tuning methods with different pre-trained models, regardless of whether Local-only or FedAvg performs better. \pengxin{Further results for FedAvg with full parameter tuning (FPT) using ViT on EMNIST and CIFAR-10 are 78.46 and 97.78, respectively. It shows HyperFL-LPM is also comparable to FPT.} This highlights the effectiveness of HyperFL-LPM.

\begin{table}[h]
% \begin{wraptable}{r}{0.55\textwidth}
\centering
% \vskip -0.08in
\resizebox{\linewidth}{!}{
\begin{tabular}{lcccccc}
\toprule 
& Arch & Local-only$^{\dagger}$ & Local-only$^{\dagger\dagger}$ & FedAvg$^{\dagger}$ & FedAvg$^{\dagger\dagger}$ & HyperFL-LPM \\

\midrule
\multirow{2}{*}{EMNIST} & ResNet & 72.83 & 80.35 & 68.99 & 75.21 & 80.32 \\
& ViT & 76.95 & 80.04 & 70.92 & 76.42 & 79.92  \\
\midrule
\multirow{2}{*}{CIFAR-10} & ResNet & 68.57 & 73.57 & 62.35 & 75.57 & 75.03  \\
& ViT & 91.82 & 89.70 & 92.32 & 95.56 & 95.40  \\
\bottomrule
\end{tabular}
}
\vskip -0.1in
\caption{The comparison of final test accuracy (\%) of different methods on various datasets with 20 clients. $^{\dagger}$ Fixed feature extractor. $^{\dagger\dagger}$ Adapter fine-tuning.}
\vskip -0.2in
\label{tab:result_2}
\end{table}
% \end{wraptable}

% % \begin{figure}[h]
% \begin{wrapfigure}{r}{0.55\textwidth}
% \vskip -0.2in
%     \centering
% \includegraphics[width=0.55\textwidth]{attack.png}
%     \vskip -0.1in
%     \caption{Reconstructed images using GIA in the CIFAR-10 dataset. PSNR is reported under each restored images.}
%     \label{fig:attack}
% \vskip -0.1in
% % \end{figure}
% \end{wrapfigure}

% \vspace{-0.5cm}
\paragraph{Privacy Evaluation.} 
% We provide experimental results to demonstrate the privacy protection capability of the HyperFL framework in this subsection and the theoretical analysis is provided in Section \ref{sec:why_work}. 
The reconstructed results of IG \cite{geiping2020inverting} are provided in Table \ref{tab:attack}, while more results of ROG \cite{yue2023gradient} and tailored attack method are provided in Table \ref{tab:attack_rog} and Figures \ref{fig:attack_rog} and \ref{fig:attack_tailored} in Appendix. From Table \ref{tab:attack}, we can observe that the native FedAvg, pFedHN, \pengxin{and pFedHN-PC} methods have a much higher risk of leaking data information (indicated by the higher PSNR and SSIM values and lower LPIPS value). This can also be seen in the reconstructed images, which closely resemble the original ones, as illustrated in Figure \ref{fig:attack} in Appendix. Although introducing DP improves data privacy, there is a significant drop in model performance, as shown in Table \ref{tab:main_result}. In contrast, HyperFL achieves a similar level of privacy protection while outperforming all DP-based methods and the native FedAvg in terms of model accuracy.

\begin{table} [h]
    % \begin{wraptable} {r} {0.4\linewidth}
        % \vskip -0.2in
        \centering
          % \vskip 0.05in
        \resizebox{0.9\linewidth}{!}{
          \begin{tabular}{l ccc ccc}
            \toprule
            & \multicolumn{3}{c}{EMNIST} & \multicolumn{3}{c}{CIFAR-10} \\
            \cmidrule(lr){2-4} \cmidrule(lr){5-7} 
            Method & PSNR  & SSIM  & LPIPS  & PSNR  & SSIM  & LPIPS  \\
            \midrule
            FedAvg & 32.64 & 0.8925 & 0.0526 & 16.16 & 0.6415 & 0.0536  \\
            pFedHN & 31.24 & 0.8701 & 0.0807 & 16.02 & 0.6351 & 0.0504 \\
            \pengxin{pFedHN-PC} & \pengxin{28.38} & \pengxin{0.8713} & \pengxin{0.0645} & \pengxin{15.80} & \pengxin{0.6247} & \pengxin{0.4407} \\
            \midrule
            DP-FedAvg & 7.74 & 0.2978 & 0.7051 & 7.90 & 0.2716 & 0.3204  \\
            CENTAUR & 9.52 & 0.2136 & 0.6712 & 9.80 & 0.2723 & 0.2882  \\
            PPSGD & 9.73 & 0.1889 & 0.6466 & 9.70 & 0.2788 & 0.2643 \\
            \midrule
            HyperFL & 7.85 & 0.3010 & 0.7147 & 8.35 & 0.2732 & 0.3132  \\
            \bottomrule
          \end{tabular}
        }
\vskip -0.1in
    \caption{Reconstruction results of IG.}
\label{tab:attack}
\vskip -0.2in
\end{table}
        % \end{wraptable}
\paragraph{Training Efficiency.}

To validate the training efficiency of the proposed HyperFL framework, we compare the training time of HyperFL with other DP-based FL methods in Table \ref{tab:training_time}. This table clearly shows the efficiency of the proposed HyperFL framework. Specifically, from this table we can see that the proposed HyperFL framework runs faster than all the compared DP-based FL methods and only slightly slower than the FedAvg method. This is because DP-based FL methods often incur additional computation cost due to their privacy-preserving mechanisms, whereas HyperFL achieves faster training by leveraging the advantages of hypernetworks, all while ensuring data privacy.

\begin{table}[h] 
% \begin{wraptable}{r}{0.5\textwidth}
% \vskip -0.2in
    \centering
    \resizebox{\linewidth}{!}{
    \begin{tabular}{|c|c|c|c|c|c|}
        \hline
         & FedAvg & DP-FedAvg & PPSGD & CENTAUR & HyperFL \\
        \hline
        \# Time (s) & 23 & 194 & 223 & 210 & 37 \\
        \hline
    \end{tabular}
    }
\vskip -0.1in
\caption{Training time of per training round on the EMNIST dataset with 20 clients of different methods.}
\vskip -0.2in
\label{tab:training_time}
\end{table}
% \end{wraptable}

% \pengxin{We further compare HyperFL-LPM with FedAvg and DP-FedAvg using ViT on EMNIST in terms of training efficiency. Training time of per training round for FedAvg-Adapter, FedAvg-FPT, DP-FedAvg-Adapter, DP-FedAvg-FPT, and HyperFL-LPM are 843s, 1015s, 1494s, 5837s, and 936s, respectively. This shows our method's efficiency for large local models. 
% % Furthermore, as pFedHN needs to generate the entire local model's parameters, it is unsuitable for large models, whereas HyperFL-LPM can efficiently manage this.
% }

\paragraph{Convergence Evaluation.}

% \begin{wrapfigure}{r}{0.5\textwidth}
% \vskip -0.25in
%     \centering
%     \subfigure[]{\includegraphics[width=0.47\linewidth]{train_loss_2.pdf} \label{fig:train_loss}}
%   %   \begin{subfigure}[h]{0.49\linewidth}
%   %   \includegraphics[width=1\linewidth]{train_loss_2.pdf}
%   %   % \caption{Training loss of different methods on the EMNIST dataset with 20 clients.}
%   %   \caption{}
%   %   \label{fig:train_loss}
%   % \end{subfigure}
%   \hfill
%   \subfigure[]{\includegraphics[width=0.47\linewidth]{f_convergence_2.pdf} \label{fig:f_convergence}}
%   % \begin{subfigure}[h]{0.49\linewidth}
%   %   \includegraphics[width=1\linewidth]{f_convergence_2.pdf}
%   %   % \caption{Parameter difference of the generated feature extractor between adjacent training round on the EMNIST dataset with 20 clients.}
%   %   \caption{}
%   %   \label{fig:f_convergence}
%   % \end{subfigure}
%   \vskip -0.1in
%   \caption{(a) Average training loss of different methods on the EMNIST dataset with 20 clients. (b) Parameter difference of the generated feature extractor of one client between adjacent training round on the EMNIST dataset with 20 clients.}
%   \label{fig:convergence}
% \vskip -0.1in
% \end{wrapfigure}

To validate to convergence of the proposed HyperFL framework, we draw the training loss of FedAvg and HyperFL in Figure \ref{fig:train_loss} and the trend of feature extractor parameters' variation in Figure \ref{fig:f_convergence}. 
From Figure \ref{fig:train_loss}, we can observe that HyperFL almost has the same convergence rate as FedAvg, which demonstrates the convergence property of HyperFL. Moreover, after convergence, the training loss of HyperFL is lower than that of FedAvg, which reflects why HyperFL performs better than FedAvg. Furthermore, in the later stages of the training process, the variation of the feature extractor parameters approaches zero, as depicted in Figure \ref{fig:f_convergence}. This further confirms the convergence property of HyperFL.

\begin{figure}[h]
\vskip -0.2in
  \centering
  % \begin{subfigure}[h]{0.47\linewidth}
  %   \includegraphics[width=0.9\linewidth]{train_loss.pdf}
  %   % \caption{Training loss of different methods on the EMNIST dataset with 20 clients.}
  %   \caption{}
  %   \label{fig:train_loss}
  % \end{subfigure}
  \subfigure[]{\includegraphics[width=0.47\linewidth]{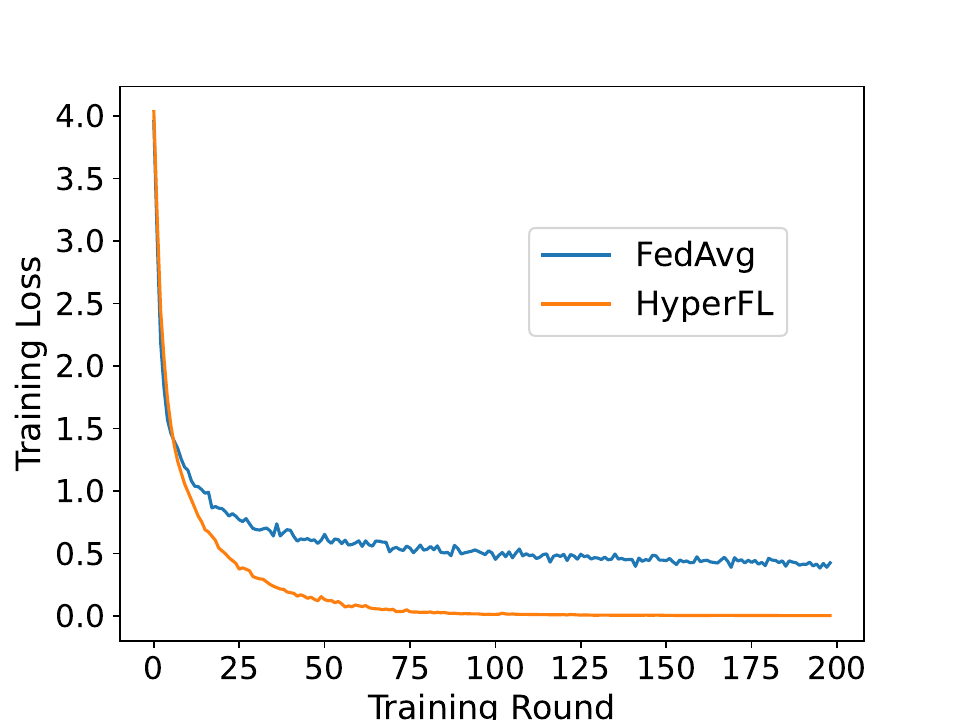} \label{fig:train_loss}}
  \hfill
  % \begin{subfigure}[h]{0.47\linewidth}
  %   \includegraphics[width=0.9\linewidth]{f_convergence.pdf}
  %   % \caption{Parameter difference of the generated feature extractor between adjacent training round on the EMNIST dataset with 20 clients.}
  %   \caption{}
  %   \label{fig:f_convergence}
  % \end{subfigure}
    \subfigure[]{\includegraphics[width=0.47\linewidth]{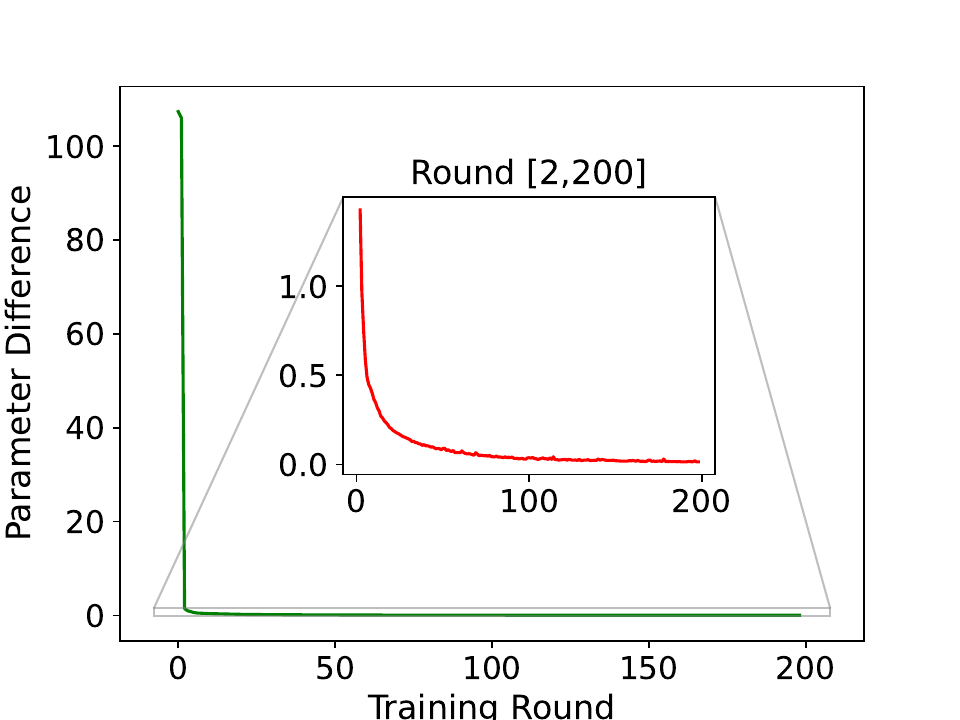} \label{fig:f_convergence}}
  \vskip -0.1in
  \caption{(a) Average training loss of different methods on the EMNIST dataset with 20 clients. (b) Parameter difference of the generated feature extractor of one client between adjacent training round on the EMNIST dataset with 20 clients.}
  \label{fig:convergence}
\vskip -0.2in
\end{figure}

\section{Conclusion}
\label{sec:conclusion}

In this paper, we propose HyperFL, a novel federated learning framework that ``breaks the direct connection'' between the shared parameters and the local private data to defend against GIA. Specifically, this framework utilizes hypernetworks to generate the parameters of the local model and only the hypernetwork parameters are uploaded to the server for aggregation to defend against GIA while without compromising performance or incurring heavy computation overhead. We hope that the proposed HyperFL framework can encourage the research community to consider the importance of developing enhanced privacy preservation FL frameworks, as an alternative to current research efforts on defense mechanisms front.

\section*{Acknowledgments}

This work was supported by National Natural Science Foundation of China (62306253, 62206075), Guangdong Natural Science Fund-General Programme (2024A1515010233), and UCSC hellman fellowship.

\bibliography{main}

\begin{thebibliography}{91}
\providecommand{\natexlab}[1]{#1}

\bibitem[{Abadi et~al.(2016)Abadi, Chu, Goodfellow, McMahan, Mironov, Talwar, and Zhang}]{abadi2016deep}
Abadi, M.; Chu, A.; Goodfellow, I.; McMahan, H.~B.; Mironov, I.; Talwar, K.; and Zhang, L. 2016.
\newblock Deep learning with differential privacy.
\newblock In \emph{Proceedings of the 2016 ACM SIGSAC conference on computer and communications security}, 308--318.

\bibitem[{Baxter(2000)}]{baxter2000model}
Baxter, J. 2000.
\newblock A model of inductive bias learning.
\newblock \emph{Journal of artificial intelligence research}, 12: 149--198.

\bibitem[{Bietti et~al.(2022)Bietti, Wei, Dudik, Langford, and Wu}]{bietti2022personalization}
Bietti, A.; Wei, C.-Y.; Dudik, M.; Langford, J.; and Wu, S. 2022.
\newblock Personalization improves privacy-accuracy tradeoffs in federated learning.
\newblock In \emph{International Conference on Machine Learning}, 1945--1962. PMLR.

\bibitem[{Bonawitz et~al.(2017)Bonawitz, Ivanov, Kreuter, Marcedone, McMahan, Patel, Ramage, Segal, and Seth}]{bonawitz2017practical}
Bonawitz, K.; Ivanov, V.; Kreuter, B.; Marcedone, A.; McMahan, H.~B.; Patel, S.; Ramage, D.; Segal, A.; and Seth, K. 2017.
\newblock Practical secure aggregation for privacy-preserving machine learning.
\newblock In \emph{proceedings of the 2017 ACM SIGSAC Conference on Computer and Communications Security}, 1175--1191.

\bibitem[{Carey, Du, and Wu(2022)}]{carey2022robust}
Carey, A.~N.; Du, W.; and Wu, X. 2022.
\newblock Robust Personalized Federated Learning under Demographic Fairness Heterogeneity.
\newblock In \emph{2022 IEEE International Conference on Big Data (Big Data)}, 1425--1434. IEEE.

\bibitem[{Cohen et~al.(2017)Cohen, Afshar, Tapson, and Van~Schaik}]{cohen2017emnist}
Cohen, G.; Afshar, S.; Tapson, J.; and Van~Schaik, A. 2017.
\newblock EMNIST: Extending MNIST to handwritten letters.
\newblock In \emph{2017 international joint conference on neural networks (IJCNN)}, 2921--2926. IEEE.

\bibitem[{Collins et~al.(2021)Collins, Hassani, Mokhtari, and Shakkottai}]{collins2021exploiting}
Collins, L.; Hassani, H.; Mokhtari, A.; and Shakkottai, S. 2021.
\newblock Exploiting shared representations for personalized federated learning.
\newblock In \emph{International conference on machine learning}, 2089--2099. PMLR.

\bibitem[{Dang et~al.(2021)Dang, Thakkar, Ramaswamy, Mathews, Chin, and Beaufays}]{dang2021revealing}
Dang, T.; Thakkar, O.; Ramaswamy, S.; Mathews, R.; Chin, P.; and Beaufays, F. 2021.
\newblock Revealing and protecting labels in distributed training.
\newblock \emph{Advances in Neural Information Processing Systems}, 34: 1727--1738.

\bibitem[{Darlow et~al.(2018)Darlow, Crowley, Antoniou, and Storkey}]{darlow2018cinic}
Darlow, L.~N.; Crowley, E.~J.; Antoniou, A.; and Storkey, A.~J. 2018.
\newblock Cinic-10 is not imagenet or cifar-10.
\newblock \emph{arXiv preprint arXiv:1810.03505}.

\bibitem[{Deng et~al.(2009)Deng, Dong, Socher, Li, Li, and Fei-Fei}]{deng2009imagenet}
Deng, J.; Dong, W.; Socher, R.; Li, L.-J.; Li, K.; and Fei-Fei, L. 2009.
\newblock Imagenet: A large-scale hierarchical image database.
\newblock In \emph{2009 IEEE conference on computer vision and pattern recognition}, 248--255. Ieee.

\bibitem[{Dosovitskiy et~al.(2021)Dosovitskiy, Beyer, Kolesnikov, Weissenborn, Zhai, Unterthiner, Dehghani, Minderer, Heigold, Gelly et~al.}]{dosovitskiy2021image}
Dosovitskiy, A.; Beyer, L.; Kolesnikov, A.; Weissenborn, D.; Zhai, X.; Unterthiner, T.; Dehghani, M.; Minderer, M.; Heigold, G.; Gelly, S.; et~al. 2021.
\newblock An Image is Worth 16x16 Words: Transformers for Image Recognition at Scale.
\newblock In \emph{International Conference on Learning Representations}.

\bibitem[{Dwork(2006)}]{dwork2006differential}
Dwork, C. 2006.
\newblock Differential privacy.
\newblock In \emph{International colloquium on automata, languages, and programming}, 1--12. Springer.

\bibitem[{Erdo{\u{g}}an, K{\"u}p{\c{c}}{\"u}, and {\c{C}}i{\c{c}}ek(2022)}]{erdougan2022unsplit}
Erdo{\u{g}}an, E.; K{\"u}p{\c{c}}{\"u}, A.; and {\c{C}}i{\c{c}}ek, A.~E. 2022.
\newblock Unsplit: Data-oblivious model inversion, model stealing, and label inference attacks against split learning.
\newblock In \emph{Proceedings of the 21st Workshop on Privacy in the Electronic Society}, 115--124.

\bibitem[{Fredrikson, Jha, and Ristenpart(2015)}]{fredrikson2015model}
Fredrikson, M.; Jha, S.; and Ristenpart, T. 2015.
\newblock Model inversion attacks that exploit confidence information and basic countermeasures.
\newblock In \emph{Proceedings of the 22nd ACM SIGSAC conference on computer and communications security}, 1322--1333.

\bibitem[{Geiping et~al.(2020)Geiping, Bauermeister, Dr{\"o}ge, and Moeller}]{geiping2020inverting}
Geiping, J.; Bauermeister, H.; Dr{\"o}ge, H.; and Moeller, M. 2020.
\newblock Inverting gradients-how easy is it to break privacy in federated learning?
\newblock \emph{Advances in Neural Information Processing Systems}, 33: 16937--16947.

\bibitem[{Geng et~al.(2023)Geng, Mou, Li, Li, Beyan, Decker, and Rong}]{geng2023improved}
Geng, J.; Mou, Y.; Li, Q.; Li, F.; Beyan, O.; Decker, S.; and Rong, C. 2023.
\newblock Improved Gradient Inversion Attacks and Defenses in Federated Learning.
\newblock \emph{IEEE Transactions on Big Data}.

\bibitem[{Gentry(2009)}]{gentry2009fully}
Gentry, C. 2009.
\newblock \emph{A fully homomorphic encryption scheme}.
\newblock Stanford university.

\bibitem[{Geyer, Klein, and Nabi(2017)}]{geyer2017differentially}
Geyer, R.~C.; Klein, T.; and Nabi, M. 2017.
\newblock Differentially private federated learning: A client level perspective.
\newblock \emph{arXiv preprint arXiv:1712.07557}.

\bibitem[{Guo et~al.(2024)Guo, Zeng, Wang, Fan, Wang, and Qu}]{guo2024selective}
Guo, P.; Zeng, S.; Wang, Y.; Fan, H.; Wang, F.; and Qu, L. 2024.
\newblock Selective Aggregation for Low-Rank Adaptation in Federated Learning.
\newblock \emph{arXiv preprint arXiv:2410.01463}.

\bibitem[{Ha, Dai, and Le(2017)}]{david2017hypernetworks}
Ha, D.; Dai, A.~M.; and Le, Q.~V. 2017.
\newblock HyperNetworks.
\newblock In \emph{The 5th International Conference on Learning Representations}.

\bibitem[{Hatamizadeh et~al.(2023)Hatamizadeh, Yin, Molchanov, Myronenko, Li, Dogra, Feng, Flores, Kautz, Xu et~al.}]{hatamizadeh2023gradient}
Hatamizadeh, A.; Yin, H.; Molchanov, P.; Myronenko, A.; Li, W.; Dogra, P.; Feng, A.; Flores, M.~G.; Kautz, J.; Xu, D.; et~al. 2023.
\newblock Do gradient inversion attacks make federated learning unsafe?
\newblock \emph{IEEE Transactions on Medical Imaging}.

\bibitem[{He et~al.(2022)He, Chen, Xie, Li, Doll{\'a}r, and Girshick}]{he2022masked}
He, K.; Chen, X.; Xie, S.; Li, Y.; Doll{\'a}r, P.; and Girshick, R. 2022.
\newblock Masked autoencoders are scalable vision learners.
\newblock In \emph{Proceedings of the IEEE/CVF conference on computer vision and pattern recognition}, 16000--16009.

\bibitem[{He et~al.(2016)He, Zhang, Ren, and Sun}]{he2016deep}
He, K.; Zhang, X.; Ren, S.; and Sun, J. 2016.
\newblock Deep residual learning for image recognition.
\newblock In \emph{Proceedings of the IEEE conference on computer vision and pattern recognition}, 770--778.

\bibitem[{He, Zhang, and Lee(2019)}]{he2019model}
He, Z.; Zhang, T.; and Lee, R.~B. 2019.
\newblock Model inversion attacks against collaborative inference.
\newblock In \emph{Proceedings of the 35th Annual Computer Security Applications Conference}, 148--162.

\bibitem[{He, Zhang, and Lee(2020)}]{he2020attacking}
He, Z.; Zhang, T.; and Lee, R.~B. 2020.
\newblock Attacking and protecting data privacy in edge--cloud collaborative inference systems.
\newblock \emph{IEEE Internet of Things Journal}, 8(12): 9706--9716.

\bibitem[{Hore and Ziou(2010)}]{hore2010image}
Hore, A.; and Ziou, D. 2010.
\newblock Image quality metrics: PSNR vs. SSIM.
\newblock In \emph{2010 20th international conference on pattern recognition}, 2366--2369. IEEE.

\bibitem[{Houlsby et~al.(2019)Houlsby, Giurgiu, Jastrzebski, Morrone, De~Laroussilhe, Gesmundo, Attariyan, and Gelly}]{houlsby2019parameter}
Houlsby, N.; Giurgiu, A.; Jastrzebski, S.; Morrone, B.; De~Laroussilhe, Q.; Gesmundo, A.; Attariyan, M.; and Gelly, S. 2019.
\newblock Parameter-efficient transfer learning for NLP.
\newblock In \emph{International Conference on Machine Learning}, 2790--2799. PMLR.

\bibitem[{Hu et~al.(2022)Hu, Wallis, Allen-Zhu, Li, Wang, Wang, Chen et~al.}]{hu2022lora}
Hu, E.~J.; Wallis, P.; Allen-Zhu, Z.; Li, Y.; Wang, S.; Wang, L.; Chen, W.; et~al. 2022.
\newblock LoRA: Low-Rank Adaptation of Large Language Models.
\newblock In \emph{International Conference on Learning Representations}.

\bibitem[{Huang et~al.(2021)Huang, Gupta, Song, Li, and Arora}]{huang2021evaluating}
Huang, Y.; Gupta, S.; Song, Z.; Li, K.; and Arora, S. 2021.
\newblock Evaluating gradient inversion attacks and defenses in federated learning.
\newblock \emph{Advances in Neural Information Processing Systems}, 34: 7232--7241.

\bibitem[{Jia et~al.(2022)Jia, Tang, Chen, Cardie, Belongie, Hariharan, and Lim}]{jia2022visual}
Jia, M.; Tang, L.; Chen, B.-C.; Cardie, C.; Belongie, S.; Hariharan, B.; and Lim, S.-N. 2022.
\newblock Visual prompt tuning.
\newblock In \emph{European Conference on Computer Vision}, 709--727. Springer.

\bibitem[{Jiang, Zhou, and Grossklags(2022)}]{jiang2022comprehensive}
Jiang, X.; Zhou, X.; and Grossklags, J. 2022.
\newblock Comprehensive analysis of privacy leakage in vertical federated learning during prediction.
\newblock \emph{Proceedings on Privacy Enhancing Technologies}.

\bibitem[{Karimi, Nutini, and Schmidt(2016)}]{karimi2016linear}
Karimi, H.; Nutini, J.; and Schmidt, M. 2016.
\newblock Linear convergence of gradient and proximal-gradient methods under the polyak-{\l}ojasiewicz condition.
\newblock In \emph{Machine Learning and Knowledge Discovery in Databases: European Conference, ECML PKDD 2016, Riva del Garda, Italy, September 19-23, 2016, Proceedings, Part I 16}, 795--811. Springer.

\bibitem[{Karimireddy et~al.(2020)Karimireddy, Kale, Mohri, Reddi, Stich, and Suresh}]{karimireddy2020scaffold}
Karimireddy, S.~P.; Kale, S.; Mohri, M.; Reddi, S.; Stich, S.; and Suresh, A.~T. 2020.
\newblock Scaffold: Stochastic controlled averaging for federated learning.
\newblock In \emph{International conference on machine learning}, 5132--5143. PMLR.

\bibitem[{Kariyappa et~al.(2023)Kariyappa, Guo, Maeng, Xiong, Suh, Qureshi, and Lee}]{kariyappa2023cocktail}
Kariyappa, S.; Guo, C.; Maeng, K.; Xiong, W.; Suh, G.~E.; Qureshi, M.~K.; and Lee, H.-H.~S. 2023.
\newblock Cocktail party attack: Breaking aggregation-based privacy in federated learning using independent component analysis.
\newblock In \emph{International Conference on Machine Learning}, 15884--15899. PMLR.

\bibitem[{Kingma and Ba(2015)}]{kingma2015adam}
Kingma, D.~P.; and Ba, J. 2015.
\newblock Adam: A method for stochastic optimization.
\newblock \emph{ICLR}.

\bibitem[{Krizhevsky, Hinton et~al.(2009)}]{krizhevsky2009learning}
Krizhevsky, A.; Hinton, G.; et~al. 2009.
\newblock Learning multiple layers of features from tiny images.

\bibitem[{Krizhevsky, Sutskever, and Hinton(2012)}]{krizhevsky2012imagenet}
Krizhevsky, A.; Sutskever, I.; and Hinton, G.~E. 2012.
\newblock Imagenet classification with deep convolutional neural networks.
\newblock \emph{Advances in neural information processing systems}, 25.

\bibitem[{Li et~al.(2023{\natexlab{a}})Li, Gu, Chen, Li, Wu, Ruan, Si, and Fan}]{li2023temporal}
Li, B.; Gu, H.; Chen, R.; Li, J.; Wu, C.; Ruan, N.; Si, X.; and Fan, L. 2023{\natexlab{a}}.
\newblock Temporal Gradient Inversion Attacks with Robust Optimization.
\newblock \emph{arXiv preprint arXiv:2306.07883}.

\bibitem[{Li et~al.(2023{\natexlab{b}})Li, Cai, Wang, Tang, Ding, Lin, and Shi}]{li2023fedtp}
Li, H.; Cai, Z.; Wang, J.; Tang, J.; Ding, W.; Lin, C.-T.; and Shi, Y. 2023{\natexlab{b}}.
\newblock FedTP: Federated Learning by Transformer Personalization.
\newblock \emph{IEEE Transactions on Neural Networks and Learning Systems}.

\bibitem[{Li et~al.(2020{\natexlab{a}})Li, Sahu, Talwalkar, and Smith}]{li2020bfederated}
Li, T.; Sahu, A.~K.; Talwalkar, A.; and Smith, V. 2020{\natexlab{a}}.
\newblock Federated learning: Challenges, methods, and future directions.
\newblock \emph{IEEE signal processing magazine}, 37(3): 50--60.

\bibitem[{Li et~al.(2020{\natexlab{b}})Li, Huang, Yang, Wang, and Zhang}]{li2019convergence}
Li, X.; Huang, K.; Yang, W.; Wang, S.; and Zhang, Z. 2020{\natexlab{b}}.
\newblock On the Convergence of FedAvg on Non-IID Data.
\newblock In \emph{International Conference on Learning Representations}.

\bibitem[{Li et~al.(2022{\natexlab{a}})Li, Wang, Chen, Zhang, Shafiq, and Gu}]{li2022e2egi}
Li, Z.; Wang, L.; Chen, G.; Zhang, Z.; Shafiq, M.; and Gu, Z. 2022{\natexlab{a}}.
\newblock E2EGI: End-to-End Gradient Inversion in Federated Learning.
\newblock \emph{IEEE Journal of Biomedical and Health Informatics}, 27(2): 756--767.

\bibitem[{Li et~al.(2022{\natexlab{b}})Li, Zhang, Liu, and Liu}]{li2022auditing}
Li, Z.; Zhang, J.; Liu, L.; and Liu, J. 2022{\natexlab{b}}.
\newblock Auditing privacy defenses in federated learning via generative gradient leakage.
\newblock In \emph{Proceedings of the IEEE/CVF Conference on Computer Vision and Pattern Recognition}, 10132--10142.

\bibitem[{Lin et~al.(2023)Lin, Wang, Li, and Shen}]{lin2023federated}
Lin, Y.; Wang, H.; Li, W.; and Shen, J. 2023.
\newblock Federated learning with hyper-network—A case study on whole slide image analysis.
\newblock \emph{Scientific Reports}, 13(1): 1724.

\bibitem[{Liu et~al.(2022)Liu, Guo, Yang, Fan, Lam, and Zhao}]{liu2022privacy}
Liu, Z.; Guo, J.; Yang, W.; Fan, J.; Lam, K.-Y.; and Zhao, J. 2022.
\newblock Privacy-preserving aggregation in federated learning: A survey.
\newblock \emph{IEEE Transactions on Big Data}.

\bibitem[{Liu et~al.(2021)Liu, Lin, Cao, Hu, Wei, Zhang, Lin, and Guo}]{liu2021swin}
Liu, Z.; Lin, Y.; Cao, Y.; Hu, H.; Wei, Y.; Zhang, Z.; Lin, S.; and Guo, B. 2021.
\newblock Swin transformer: Hierarchical vision transformer using shifted windows.
\newblock In \emph{Proceedings of the IEEE/CVF international conference on computer vision}, 10012--10022.

\bibitem[{Lowy and Razaviyayn(2023)}]{lowy2023private}
Lowy, A.; and Razaviyayn, M. 2023.
\newblock Private Federated Learning Without a Trusted Server: Optimal Algorithms for Convex Losses.
\newblock In \emph{The Eleventh International Conference on Learning Representations}.

\bibitem[{Luo et~al.(2022)Luo, Zhu, Fang, Kou, Hou, and Wang}]{luo2022effective}
Luo, Z.; Zhu, C.; Fang, L.; Kou, G.; Hou, R.; and Wang, X. 2022.
\newblock An effective and practical gradient inversion attack.
\newblock \emph{International Journal of Intelligent Systems}, 37(11): 9373--9389.

\bibitem[{Ma et~al.(2022)Ma, Naas, Sigg, and Lyu}]{ma2022privacy}
Ma, J.; Naas, S.-A.; Sigg, S.; and Lyu, X. 2022.
\newblock Privacy-preserving federated learning based on multi-key homomorphic encryption.
\newblock \emph{International Journal of Intelligent Systems}, 37(9): 5880--5901.

\bibitem[{Ma et~al.(2023)Ma, Sun, Cui, Li, Guan, and Liu}]{ma2023instance}
Ma, K.; Sun, Y.; Cui, J.; Li, D.; Guan, Z.; and Liu, J. 2023.
\newblock Instance-wise Batch Label Restoration via Gradients in Federated Learning.
\newblock In \emph{The Eleventh International Conference on Learning Representations}.

\bibitem[{McMahan et~al.(2017)McMahan, Moore, Ramage, Hampson, and y~Arcas}]{mcmahan2017communication}
McMahan, B.; Moore, E.; Ramage, D.; Hampson, S.; and y~Arcas, B.~A. 2017.
\newblock Communication-efficient learning of deep networks from decentralized data.
\newblock In \emph{Artificial intelligence and statistics}, 1273--1282. PMLR.

\bibitem[{McMahan et~al.(2018)McMahan, Ramage, Talwar, and Zhang}]{mcmahan2018learning}
McMahan, H.~B.; Ramage, D.; Talwar, K.; and Zhang, L. 2018.
\newblock Learning Differentially Private Recurrent Language Models.
\newblock In \emph{International Conference on Learning Representations}.

\bibitem[{Mou et~al.(2021)Mou, Fu, Lei, and Hu}]{mou2021verifiable}
Mou, W.; Fu, C.; Lei, Y.; and Hu, C. 2021.
\newblock A verifiable federated learning scheme based on secure multi-party computation.
\newblock In \emph{International Conference on Wireless Algorithms, Systems, and Applications}, 198--209. Springer.

\bibitem[{Mugunthan et~al.(2019)Mugunthan, Polychroniadou, Byrd, and Balch}]{mugunthan2019smpai}
Mugunthan, V.; Polychroniadou, A.; Byrd, D.; and Balch, T.~H. 2019.
\newblock Smpai: Secure multi-party computation for federated learning.
\newblock In \emph{Proceedings of the NeurIPS 2019 Workshop on Robust AI in Financial Services}, 1--9. MIT Press Cambridge, MA, USA.

\bibitem[{Nair and Hinton(2010)}]{nair2010rectified}
Nair, V.; and Hinton, G.~E. 2010.
\newblock Rectified linear units improve restricted boltzmann machines.
\newblock In \emph{Proceedings of the 27th international conference on machine learning (ICML-10)}, 807--814.

\bibitem[{Park and Lim(2022)}]{park2022privacy}
Park, J.; and Lim, H. 2022.
\newblock Privacy-preserving federated learning using homomorphic encryption.
\newblock \emph{Applied Sciences}, 12(2): 734.

\bibitem[{Phong et~al.(2017)Phong, Aono, Hayashi, Wang, and Moriai}]{phong2017privacy}
Phong, L.~T.; Aono, Y.; Hayashi, T.; Wang, L.; and Moriai, S. 2017.
\newblock Privacy-preserving deep learning: Revisited and enhanced.
\newblock In \emph{Applications and Techniques in Information Security: 8th International Conference, ATIS 2017, Auckland, New Zealand, July 6--7, 2017, Proceedings}, 100--110. Springer.

\bibitem[{Qu et~al.(2022)Qu, Zhou, Liang, Xia, Wang, Adeli, Fei-Fei, and Rubin}]{qu2022rethinking}
Qu, L.; Zhou, Y.; Liang, P.~P.; Xia, Y.; Wang, F.; Adeli, E.; Fei-Fei, L.; and Rubin, D. 2022.
\newblock Rethinking architecture design for tackling data heterogeneity in federated learning.
\newblock In \emph{Proceedings of the IEEE/CVF conference on computer vision and pattern recognition}, 10061--10071.

\bibitem[{Ren et~al.(2023)Ren, Deng, Xie, Ma, and Ma}]{ren2023gradient}
Ren, H.; Deng, J.; Xie, X.; Ma, X.; and Ma, J. 2023.
\newblock Gradient leakage defense with key-lock module for federated learning.
\newblock \emph{arXiv preprint arXiv:2305.04095}.

\bibitem[{Ruder(2016)}]{ruder2016overview}
Ruder, S. 2016.
\newblock An overview of gradient descent optimization algorithms.
\newblock \emph{arXiv preprint arXiv:1609.04747}.

\bibitem[{Russakovsky et~al.(2015)Russakovsky, Deng, Su, Krause, Satheesh, Ma, Huang, Karpathy, Khosla, Bernstein et~al.}]{russakovsky2015imagenet}
Russakovsky, O.; Deng, J.; Su, H.; Krause, J.; Satheesh, S.; Ma, S.; Huang, Z.; Karpathy, A.; Khosla, A.; Bernstein, M.; et~al. 2015.
\newblock Imagenet large scale visual recognition challenge.
\newblock \emph{International journal of computer vision}, 115: 211--252.

\bibitem[{Scheliga, M{\"a}der, and Seeland(2022)}]{scheliga2022precode}
Scheliga, D.; M{\"a}der, P.; and Seeland, M. 2022.
\newblock Precode-a generic model extension to prevent deep gradient leakage.
\newblock In \emph{Proceedings of the IEEE/CVF Winter Conference on Applications of Computer Vision}, 1849--1858.

\bibitem[{Shamsian et~al.(2021)Shamsian, Navon, Fetaya, and Chechik}]{shamsian2021personalized}
Shamsian, A.; Navon, A.; Fetaya, E.; and Chechik, G. 2021.
\newblock Personalized federated learning using hypernetworks.
\newblock In \emph{International Conference on Machine Learning}, 9489--9502. PMLR.

\bibitem[{Shen et~al.(2023)Shen, Ye, Kang, Hassani, and Shokri}]{shen2023share}
Shen, Z.; Ye, J.; Kang, A.; Hassani, H.; and Shokri, R. 2023.
\newblock Share your representation only: Guaranteed improvement of the privacy-utility tradeoff in federated learning.
\newblock \emph{The Eleventh International Conference on Learning Representations}.

\bibitem[{Sun et~al.(2020)Sun, Li, Wang, Yang, Li, and Chen}]{sun2020provable}
Sun, J.; Li, A.; Wang, B.; Yang, H.; Li, H.; and Chen, Y. 2020.
\newblock Provable defense against privacy leakage in federated learning from representation perspective.
\newblock \emph{arXiv preprint arXiv:2012.06043}.

\bibitem[{Tashakori et~al.(2023)Tashakori, Zhang, Wang, and Servati}]{tashakori2023semipfl}
Tashakori, A.; Zhang, W.; Wang, Z.~J.; and Servati, P. 2023.
\newblock SemiPFL: personalized semi-supervised federated learning framework for edge intelligence.
\newblock \emph{IEEE Internet of Things Journal}.

\bibitem[{Van~der Maaten and Hinton(2008)}]{van2008visualizing}
Van~der Maaten, L.; and Hinton, G. 2008.
\newblock Visualizing data using t-SNE.
\newblock \emph{Journal of machine learning research}, 9(11).

\bibitem[{Vaswani et~al.(2017)Vaswani, Shazeer, Parmar, Uszkoreit, Jones, Gomez, Kaiser, and Polosukhin}]{vaswani2017attention}
Vaswani, A.; Shazeer, N.; Parmar, N.; Uszkoreit, J.; Jones, L.; Gomez, A.~N.; Kaiser, {\L}.; and Polosukhin, I. 2017.
\newblock Attention is all you need.
\newblock \emph{Advances in neural information processing systems}, 30.

\bibitem[{Wang, Liang, and He(2024)}]{wang2024towards}
Wang, Y.; Liang, J.; and He, R. 2024.
\newblock Towards Eliminating Hard Label Constraints in Gradient Inversion Attacks.
\newblock In \emph{The Twelfth International Conference on Learning Representations}.

\bibitem[{Wang et~al.(2023)Wang, Wang, Dinh, Du, and Xu}]{wang2023learning}
Wang, Y.; Wang, X.; Dinh, A.-D.; Du, B.; and Xu, C. 2023.
\newblock Learning to Schedule in Diffusion Probabilistic Models.
\newblock In \emph{Proceedings of the 29th ACM SIGKDD Conference on Knowledge Discovery and Data Mining}.

\bibitem[{Wang et~al.(2004)Wang, Bovik, Sheikh, and Simoncelli}]{wang2004image}
Wang, Z.; Bovik, A.~C.; Sheikh, H.~R.; and Simoncelli, E.~P. 2004.
\newblock Image quality assessment: from error visibility to structural similarity.
\newblock \emph{IEEE transactions on image processing}, 13(4): 600--612.

\bibitem[{Wei et~al.(2020)Wei, Liu, Loper, Chow, Gursoy, Truex, and Wu}]{wei2020framework}
Wei, W.; Liu, L.; Loper, M.; Chow, K.-H.; Gursoy, M.~E.; Truex, S.; and Wu, Y. 2020.
\newblock A framework for evaluating gradient leakage attacks in federated learning.
\newblock \emph{arXiv preprint arXiv:2004.10397}.

\bibitem[{Xiao, Rasul, and Vollgraf(2017)}]{xiao2017fashion}
Xiao, H.; Rasul, K.; and Vollgraf, R. 2017.
\newblock Fashion-mnist: a novel image dataset for benchmarking machine learning algorithms.
\newblock \emph{arXiv preprint arXiv:1708.07747}.

\bibitem[{Xu et~al.(2015)Xu, Wang, Chen, and Li}]{xu2015empirical}
Xu, B.; Wang, N.; Chen, T.; and Li, M. 2015.
\newblock Empirical evaluation of rectified activations in convolutional network.
\newblock \emph{arXiv preprint arXiv:1505.00853}.

\bibitem[{Xu, Tong, and Huang(2023)}]{xu2023personalized}
Xu, J.; Tong, X.; and Huang, S.-L. 2023.
\newblock Personalized Federated Learning with Feature Alignment and Classifier Collaboration.
\newblock In \emph{The Eleventh International Conference on Learning Representations}.

\bibitem[{Xu et~al.(2022)Xu, Peng, Tan, Tian, Ma, and Niu}]{xu2022non}
Xu, Y.; Peng, C.; Tan, W.; Tian, Y.; Ma, M.; and Niu, K. 2022.
\newblock Non-interactive verifiable privacy-preserving federated learning.
\newblock \emph{Future Generation Computer Systems}, 128: 365--380.

\bibitem[{Yao(1982)}]{yao1982protocols}
Yao, A.~C. 1982.
\newblock Protocols for secure computations.
\newblock In \emph{23rd annual symposium on foundations of computer science (sfcs 1982)}, 160--164. IEEE.

\bibitem[{Yin et~al.(2021)Yin, Mallya, Vahdat, Alvarez, Kautz, and Molchanov}]{yin2021see}
Yin, H.; Mallya, A.; Vahdat, A.; Alvarez, J.~M.; Kautz, J.; and Molchanov, P. 2021.
\newblock See through gradients: Image batch recovery via gradinversion.
\newblock In \emph{Proceedings of the IEEE/CVF Conference on Computer Vision and Pattern Recognition}, 16337--16346.

\bibitem[{Yu, Yang, and Zhu(2019)}]{yu2019parallel}
Yu, H.; Yang, S.; and Zhu, S. 2019.
\newblock Parallel restarted SGD with faster convergence and less communication: Demystifying why model averaging works for deep learning.
\newblock In \emph{Proceedings of the AAAI Conference on Artificial Intelligence}, volume~33, 5693--5700.

\bibitem[{Yu, Bagdasaryan, and Shmatikov(2020)}]{yu2020salvaging}
Yu, T.; Bagdasaryan, E.; and Shmatikov, V. 2020.
\newblock Salvaging federated learning by local adaptation.
\newblock \emph{arXiv preprint arXiv:2002.04758}.

\bibitem[{Yue et~al.(2023)Yue, Jin, Wong, Baron, and Dai}]{yue2023gradient}
Yue, K.; Jin, R.; Wong, C.-W.; Baron, D.; and Dai, H. 2023.
\newblock Gradient obfuscation gives a false sense of security in federated learning.
\newblock In \emph{32nd USENIX Security Symposium (USENIX Security 23)}, 6381--6398.

\bibitem[{Zeng et~al.(2024)Zeng, Guo, Wang, Wang, Zhou, and Qu}]{zeng2024tackling}
Zeng, S.; Guo, P.; Wang, S.; Wang, J.; Zhou, Y.; and Qu, L. 2024.
\newblock Tackling data heterogeneity in federated learning via loss decomposition.
\newblock In \emph{International Conference on Medical Image Computing and Computer-Assisted Intervention}, 707--717. Springer.

\bibitem[{Zhang et~al.(2020{\natexlab{a}})Zhang, Li, Xia, Wang, Yan, and Liu}]{zhang2020batchcrypt}
Zhang, C.; Li, S.; Xia, J.; Wang, W.; Yan, F.; and Liu, Y. 2020{\natexlab{a}}.
\newblock $\{$BatchCrypt$\}$: Efficient homomorphic encryption for $\{$Cross-Silo$\}$ federated learning.
\newblock In \emph{2020 USENIX annual technical conference (USENIX ATC 20)}, 493--506.

\bibitem[{Zhang et~al.(2024)Zhang, Zeng, Zhang, Wang, Wang, Zhou, Liang, and Qu}]{zhang2024flhetbench}
Zhang, J.; Zeng, S.; Zhang, M.; Wang, R.; Wang, F.; Zhou, Y.; Liang, P.~P.; and Qu, L. 2024.
\newblock FLHetBench: Benchmarking Device and State Heterogeneity in Federated Learning.
\newblock In \emph{Proceedings of the IEEE/CVF Conference on Computer Vision and Pattern Recognition}, 12098--12108.

\bibitem[{Zhang et~al.(2019)Zhang, Lucas, Ba, and Hinton}]{zhang2019lookahead}
Zhang, M.; Lucas, J.; Ba, J.; and Hinton, G.~E. 2019.
\newblock Lookahead optimizer: k steps forward, 1 step back.
\newblock \emph{Advances in neural information processing systems}, 32.

\bibitem[{Zhang et~al.(2021)Zhang, Sapra, Fidler, Yeung, and Alvarez}]{zhang2021personalized}
Zhang, M.; Sapra, K.; Fidler, S.; Yeung, S.; and Alvarez, J.~M. 2021.
\newblock Personalized Federated Learning with First Order Model Optimization.
\newblock In \emph{International Conference on Learning Representations}.

\bibitem[{Zhang et~al.(2018)Zhang, Isola, Efros, Shechtman, and Wang}]{zhang2018unreasonable}
Zhang, R.; Isola, P.; Efros, A.~A.; Shechtman, E.; and Wang, O. 2018.
\newblock The unreasonable effectiveness of deep features as a perceptual metric.
\newblock In \emph{Proceedings of the IEEE conference on computer vision and pattern recognition}, 586--595.

\bibitem[{Zhang et~al.(2020{\natexlab{b}})Zhang, Fu, Wang, Zhou, and Chen}]{zhang2020privacy}
Zhang, X.; Fu, A.; Wang, H.; Zhou, C.; and Chen, Z. 2020{\natexlab{b}}.
\newblock A privacy-preserving and verifiable federated learning scheme.
\newblock In \emph{ICC 2020-2020 IEEE International Conference on Communications (ICC)}, 1--6. IEEE.

\bibitem[{Zhao, Mopuri, and Bilen(2020)}]{zhao2020idlg}
Zhao, B.; Mopuri, K.~R.; and Bilen, H. 2020.
\newblock idlg: Improved deep leakage from gradients.
\newblock \emph{arXiv preprint arXiv:2001.02610}.

\bibitem[{Zhu and Blaschko(2021)}]{zhu2021r}
Zhu, J.; and Blaschko, M.~B. 2021.
\newblock R-GAP: Recursive Gradient Attack on Privacy.
\newblock In \emph{International Conference on Learning Representations}.

\bibitem[{Zhu, Liu, and Han(2019)}]{zhu2019deep}
Zhu, L.; Liu, Z.; and Han, S. 2019.
\newblock Deep leakage from gradients.
\newblock \emph{Advances in neural information processing systems}, 32.

\end{thebibliography}

\clearpage
\setcounter{section}{0}
\renewcommand{\thesection}{\Alph{section}}

\section{Proofs of Theoretical Results} \label{sec:proof}
\subsection{Proof of Theorem \ref{theorem1}}
\label{convergence_proof}
\begin{proof}
% \textit{Proof of Theorem \ref{theorem1}.} 
Let $\phi_i^t, \varphi_i^t, \mathbf{v}_i^t$ be the model parameters maintained in the $i$-th client at the $t$-th step. Let $\mathcal{I}_E$ be the set of global synchronization steps, i.e., $\mathcal{I}_E=\{nE \mid n=1,2,\cdots\}$. If $t+1 \in \mathcal{I}_E$, which represents the time step for communication, then the one-step update of HyperFL can be described as follows:
% {\small
\begin{equation*}
\begin{array}{c}
\left(\begin{array}{c}
\phi_i^t \\
\varphi_i^t \\
\mathbf{v}_i^t
\end{array}\right) \underset{\text{SGD  of } \phi_i^t}{\longrightarrow}\left(\begin{array}{c}
\phi_i^{t+1} \\
\varphi_i^t \\
\mathbf{v}_i^t
\end{array}\right) \underset{\text{SGD of }\varphi_i^t, \mathbf{v}_i^t}{\longrightarrow}\left(\begin{array}{c}
\phi_i^{t+1} \\
\varphi_i^{t+1} \\
\mathbf{v}_i^{t+1}
\end{array}\right) \\
\\\underset{\text { if } t+1 \in \mathcal{I}_E}{\longrightarrow}\left(\begin{array}{c}
\phi_i^{t+1} \\
\sum_{j=1}^m w_j \varphi_j^{t+1} \\
\mathbf{v}_i^{t+1}
\end{array}\right).
\end{array}
\end{equation*}
% }
For convenience, we denote the parameters in each sub-step above as follows:
% {\small
\begin{equation*}
\begin{array}{c}
x_i^t = \left(\begin{array}{c}
\phi_i^t \\
\varphi_i^t \\
\mathbf{v}_i^t
\end{array}\right),
y_i^t = \left(\begin{array}{c}
\phi_i^{t+1} \\
\varphi_i^{t} \\
\mathbf{v}_i^{t}
\end{array}\right), \\
\\
x_i^{t+1,1} = \left(\begin{array}{c}
\phi_i^{t+1} \\
\varphi_i^{t+1} \\
\mathbf{v}_i^{t+1}
\end{array}\right),
x_i^{t+1,2} = \left(\begin{array}{c}
\phi_i^{t+1} \\
\sum_{j=1}^m w_j \varphi_j^{t+1} \\
\mathbf{v}_i^{t+1}
\end{array}\right),
\end{array}
\end{equation*}
% }
\begin{equation*}
x_i^{t+1}= \begin{cases}x_i^{t+1, 1} & \text { if } t+1 \notin \mathcal{I}_E, \\ x_i^{t+1, 2} & \text { if } t+1 \in \mathcal{I}_E.\end{cases}
\end{equation*}
Here, the variable $x_i^{t+1, 1}$ represents the immediate result of one sub-step SGD update from the parameter of the previous sub-step $y_i^t$, and $x_i^{t+1, 2}$ represents the parameter obtained after communication steps (if possible). 
% Therefore, we can only fetch $x_i^{t+1, 2}$ when $t+1 \in \mathcal{I}_E$.
\\
Furthermore, we denote the learning rate and stochastic gradient of step $t$ as follows:
% {\small
\begin{equation*}
\eta = \left(\begin{array}{c}
\eta_g \\
\eta_h \\
\eta_{v}
\end{array}\right),
\end{equation*}
\begin{equation*}
g_i^t=
\left(\begin{array}{c}
g_{i,\phi}^t \\
g_{i,\varphi}^t \\
g_{i,\mathbf{v}}^t
\end{array}\right)
= \left(\begin{array}{l}
\nabla_\phi \mathcal{L}_i\left(h\left(\mathbf{v}_i^t, \varphi_i^t\right), \phi_i^t, \xi_i^t\right) \\
\nabla_{\varphi} \mathcal{L}_i\left(h\left(\mathbf{v}_i^t, \varphi_i^t\right), \phi_i^{t+1}, \xi_i^t\right) \\
\nabla_\mathbf{v} \mathcal{L}_i\left(h\left(\mathbf{v}_i^t, \varphi_i^t\right), \phi_i^{t+1}, \xi_i^t\right)
\end{array}\right),
\end{equation*}
\begin{equation*}
\bar{g}_i^t=
\left(\begin{array}{c}
\bar{g}_{i,\phi}^t \\
\bar{g}_{i,\varphi}^t \\
\bar{g}_{i,\mathbf{v}}^t
\end{array}\right)
= \left(\begin{array}{l}
\nabla_\phi \mathcal{L}_i\left(h\left(\mathbf{v}_i^t, \varphi_i^t\right), \phi_i^t\right) \\
\nabla_{\varphi} \mathcal{L}_i\left(h\left(\mathbf{v}_i^t, \varphi_i^t\right), \phi_i^{t+1}\right) \\
\nabla_\mathbf{v} \mathcal{L}_i\left(h\left(\mathbf{v}_i^t, \varphi_i^t\right), \phi_i^{t+1}\right)
\end{array}\right),
\end{equation*}
% }
where $\xi_i^t$ is the data uniformly chosen from the local data set of client $i$ at step $t$, then $\mathbb{E}\left[g_i^t\right] = \bar{g}_i^t$.
\\
\\
Next, we apply the inequality of the smoothness Assumption \ref{assumption1} to each sub-step of the one-step update for client $i$.
Firstly, by the smoothness of $\mathcal{L}_i$, we have
% {\small
\begin{equation}\label{l-smooth1}
\mathcal{L}_i\left(y_i^t\right) \leq \mathcal{L}_i\left(x_i^t\right)+\left\langle y_i^t-x_i^t, \bar{g}_{i,\phi}^t\right\rangle+\frac{L}{2}\left\|y_i^t-x_i^t\right\|^2.
\end{equation}
% }
For the second term on the right side of inequality (\ref{l-smooth1}), according to the law of total expectation, we have
\begin{equation*}
\begin{aligned} 
\mathbb{E}\left[\left\langle y_i^t-x_i^t, \bar{g}_{i,\phi}^t\right\rangle\right]
&= \mathbb{E}\left[\left\langle -\eta_g g_{i,\phi}^t, \bar{g}_{i,\phi}^t\right\rangle\right]\\  
&= \mathbb{E}\left\{\mathbb{E}\left[\left\langle -\eta_g g_{i,\phi}^t, \bar{g}_{i,\phi}^t\right\rangle\right]|\xi_i^t\right\}\\  
&= \mathbb{E}\left\{\mathbb{E}\left[\left\langle -\eta_g g_{i,\phi}^t|\xi_i^t, \bar{g}_{i,\phi}^t\right\rangle\right]\right\}\\  
&= \mathbb{E}\left[\left\langle -\eta_g \bar{g}_{i,\phi}^t, \bar{g}_{i,\phi}^t\right\rangle\right]\\
& = -\eta_{g}\mathbb{E}\left[(\bar{g}_{i,\phi}^t)^2\right].
\end{aligned} 
\end{equation*}
For the third term on the right side of the inequality (\ref{l-smooth1}), we have
\begin{equation*}
\begin{aligned}
\mathbb{E}\left[\frac{L}{2}\left\|y_i^t-x_i^t\right\|^2\right]
&= \mathbb{E}\left[\frac{L}{2}\left\|-\eta_g g_{i,\phi}^t\right\|^2\right]\\
 &= \eta_g^2\frac{L}{2}\mathbb{E}\left[\left| g_{i,\phi}^t\right\|^2\right] \\
 &\leq \eta_g^2\frac{LG^2}{2},
\end{aligned}
\end{equation*}
where in the last inequality, we use the bounded gradient Assumption \ref{assumption3}.\\
From the above inequalities, and taking the expectation of inequality (\ref{l-smooth1}), we can get
\begin{equation}
\label{eq1}
\mathbb{E}\left[\mathcal{L}_i\left(y_i^t\right) - \mathcal{L}_i\left(x_i^t\right)\right] \leq -\eta_g \mathbb{E}\left[(\bar{g}_{i,\phi}^t)^2\right] + \eta_g^2\frac{LG^2}{2}.
\end{equation}
Secondly, by the smoothness of $\mathcal{L}_i$, we have
\begin{equation}\label{l-smooth2}
\begin{aligned}
\mathcal{L}_i\left(x_i^{t+1,1}\right) \leq &  \mathcal{L}_i\left(y_i^t\right)+\left\langle x_i^{t+1,1}-y_i^t, \bar{g}_{i,\phi}^t\right\rangle \\
&+\frac{L}{2}\left\|x_i^{t+1,1}-y_i^t\right\|^2.
\end{aligned}
\end{equation}
Similar to inequality (\ref{eq1}), taking the expectation of inequality (\ref{l-smooth2}), we get
{\small
\begin{equation}
\label{eq2}
\begin{aligned}
&\mathbb{E}\left[\mathcal{L}_i\left(x_i^{t+1,1}\right) - \mathcal{L}_i\left(y_i^t\right)\right] \\
\leq & -\eta_h\mathbb{E}\left[(\bar{g}_{i,\varphi}^t)^2\right] -\eta_{v}\mathbb{E}\left[(\bar{g}_{i,\mathbf{v}}^t)^2\right] + (\eta_h^2+\eta_{v}^2)\frac{LG^2}{2}.
\end{aligned}
\end{equation}
}Thirdly, by the smoothness of $\mathcal{L}_i$, we have
{\small
\begin{equation}\label{l-smooth3}
\begin{aligned}
\mathcal{L}_i\left(x_i^{t+1,2}\right) \leq & \mathcal{L}_i\left(x_i^{t+1,1}\right) + \left\langle x_i^{t+1,2}-x_i^{t+1,1}, \bar{g}_{i,\varphi}^t\right\rangle \\
&+\frac{L}{2}\left\|x_i^{t+1,2}-x_i^{t+1,1}\right\|^2.
\end{aligned}
\end{equation}
}From the iterative formula of SGD, it is clear that
{\small
\begin{equation}
\label{cumm_update_equa}
    \varphi_j^{t+1} = \varphi_j^{t-E+1} - \eta_h \sum_{t_0=t-E+1}^tg_{j,\varphi}^{t_0}, \quad \forall j, t+1 \in \mathcal{I}_E.
\end{equation}
}Then, for the third term on the right side of inequality (\ref{l-smooth3}), we apply the equality (\ref{cumm_update_equa}) and take the expectation, which yields
\begin{equation*}
\begin{aligned}
&\mathbb{E}\left[\frac{L}{2}\left\|x_i^{t+1,2}-x_i^{t+1,1}\right\|^2\right]\\
=  &\frac{L}{2} \mathbb{E}\left[\left\| - \eta_h\sum_{j=1}^m w_j \sum_{t_0=t-E+1}^t (g_{j,\varphi}^{t_0}-g_{i,\varphi}^{t_0}) \right\|^2\right] \\
\leq &\eta_h^2\frac{L}{2} \sum_{j=1}^m w_j\mathbb{E}\left[\left\| \sum_{t_0=t-E+1}^t (g_{j,\varphi}^{t_0}-g_{i,\varphi}^{t_0}) \right\|^2\right] \\
\leq& \eta_h^2\frac{L}{2} \sum_{j=1}^m w_j \sum_{t_0=t-E+1}^t \mathbb{E}\left[\left\| (g_{j,\varphi}^{t_0}-g_{i,\varphi}^{t_0}) \right\|^2\right]\\
\leq& \eta_h^2\frac{L}{2} \sum_{j=1}^m w_j \sum_{t_0=t-E+1}^t \mathbb{E}\left[\frac{1}{2} \left\|g_{j,\varphi}^{t_0}\right\|^2+\frac{1}{2} \left\|g_{i,\varphi}^{t_0}\right\|^2\right]\\
\leq& \eta_h^2 \frac{(E-1)LG^2}{2},
\end{aligned}
\end{equation*}
where in the last inequality, we use the bounded gradient Assumption \ref{assumption3}.\\
For the second term on the right side of inequality (\ref{l-smooth3}), we take the expectation and get
\begin{equation*}
\begin{aligned}
&\mathbb{E}\left[\left\langle x_i^{t+1,2}-x_i^{t+1,1}, \bar{g}_{i,\varphi}^t\right\rangle\right]\\
\leq & \frac{1}{2 \eta_h}\left(x_i^{t+1,2}-x_i^{t+1,1}\right)^2+\frac{1}{2} \eta_h \left(\bar{g}_{i, \varphi}^t\right)^2 \\
\leq & \frac{1}{2 \eta_h} \eta_h^2(E-1) G^2+\frac{1}{2} \eta_h \left(\bar{g}_{i, \varphi}^t\right)^2 \\
=& \eta_h \frac{(E-1) G^2}{2}+\frac{1}{2} \eta_h \left(\bar{g}_{i, \varphi}^t\right)^2,
\end{aligned}
\end{equation*}
where we use the Cauchy-Schwarz inequality and the AM-GM inequality in the first inequality, and the bounded gradient Assumption \ref{assumption3} in the second inequality above. \\
Then, based on the above inequalities and taking the expectation of inequality (\ref{l-smooth3}), we have
{\small
\begin{equation}
\label{eq3}
\begin{aligned}
& \mathbb{E}\left[\mathcal{L}_i\left(x_i^{t+1,2}\right) - \mathcal{L}_i\left(x_i^{t+1,1}\right)\right] \\
\leq &\eta_h^2 \frac{(E-1)LG^2}{2} + \eta_h \frac{(E-1) G^2}{2}+\frac{1}{2} \eta_h \left(\bar{g}_{i, \varphi}^t\right)^2.
\end{aligned}
\end{equation}
}Summing up inequalities (\ref{eq1}) and (\ref{eq2}), we get
{\small
\begin{equation}
\begin{aligned}
\mathbb{E}\left[\mathcal{L}_i\left(x_i^{t+1,1}\right) - \mathcal{L}_i\left(x_i^t\right)\right] \leq -\eta^{\top} \mathbb{E}\left[(g_i^t)^2\right] + \left\|\eta\right\|^2 \frac{LG^2}{2}.
\end{aligned}
\end{equation}
}
Summing up inequalities (\ref{eq1}), (\ref{eq2}), and (\ref{eq3}), we get
{\small
\begin{equation}
\begin{aligned}
&\mathbb{E}\left[\mathcal{L}_i\left(x_i^{t+1,2}\right) - \mathcal{L}_i\left(x_i^t\right)\right] \\
\leq  &-\left(\begin{array}{c}\eta_g \\ \frac{1}{2} \eta_h \\ \eta_v\end{array}\right)^{\top} \mathbb{E}\left[\left(\bar{g}_i^t\right)^2\right] \\
&+ \left[\eta_g^2+\eta_v^2+(E-1) \eta_h^2+\frac{E-1}{L} \eta_h\right] \frac{L G^2}{2}.
\end{aligned}
\end{equation}
}Then, let 
% $\eta_{min}$ denote the minimum value in $\{\eta_g,\frac{1}{2}\eta_h,\eta_v\}$, i.e., 
$\eta_{min} = \min\{\eta_g,\frac{1}{2}\eta_h,\eta_v\}$, we have
\begin{equation} \label{neq:eta_min}
\begin{aligned}
&\mathbb{E}\left[\mathcal{L}_i\left(x_i^{t+1,2}\right) - \mathcal{L}_i\left(x_i^t\right)\right] \\
\leq  &-\eta_{min} \mathbb{E}\left[\left\|\bar{g}_i^t\right\|^2\right] \\
&+\left[\eta_g^2+\eta_v^2+(E-1) \eta_h^2+\frac{E-1}{L} \eta_h\right] \frac{L G^2}{2}.
\end{aligned}
\end{equation}
By rewriting the above inequality (\ref{neq:eta_min}), we get
{\small
\begin{equation} \label{neq:e_g}
\begin{aligned}
\mathbb{E}\left[\left\|\bar{g}_i^t\right\|^2\right] \leq &\frac{\mathbb{E}\left[\mathcal{L}_i\left(x_i^{t}\right) - \mathcal{L}_i\left(x_i^{t+1}\right)\right]}{\eta_{min}} \\
&+ \frac{\eta_g^2+\eta_v^2+(E-1) \eta_h^2+\frac{E-1}{L} \eta_h}{\eta_{min}}\frac{LG^2}{2}.
\end{aligned}
\end{equation}
}Let $M$ be a constant that satisfies the inequality $\eta_g^2+\eta_v^2+(E-1) \eta_h^2+\frac{E-1}{L} \eta_h \leq M \eta_{min}^2$, the aforementioned inequality (\ref{neq:e_g}) can be further simplified as
{\small
\begin{equation}
\label{leq1}
\mathbb{E}\left[\left\|\bar{g}_i^t\right\|^2\right] \leq \frac{\mathbb{E}\left[\mathcal{L}_i\left(x_i^{t}\right) - \mathcal{L}_i\left(x_i^{t+1}\right)\right]}{\eta_{min}} + \eta_{min} \frac{LMG^2}{2}.
\end{equation}
}
Now, by repeatedly applying inequality (\ref{leq1}) for different values of $t$ and summing up the results, we get
\begin{equation}
\label{leq29}
\begin{aligned}
\sum_{t=1}^{T} \mathbb{E}\left[\left\|\bar{g}_i^t\right\|^2\right] \leq &\frac{\mathbb{E}\left[\mathcal{L}_i\left(x_i^{1}\right) - \mathcal{L}_i\left(x_i^*\right)\right]}{\eta_{min}} \\
&+ \eta_{min} \frac{LMG^2}{2}T.
\end{aligned}
\end{equation}
Dividing both side of inequality (\ref{leq29}) by $T$, we get
\begin{equation}
\begin{aligned}
\frac{1}{T}\sum_{t=1}^{T} \mathbb{E}\left[\left\|\bar{g}_i^t\right\|^2\right] \leq &\frac{\mathbb{E}\left[\mathcal{L}_i\left(x_i^{1}\right) - \mathcal{L}_i\left(x_i^*\right)\right]}{\eta_{min}T}\\
&+ \eta_{min} \frac{LMG^2}{2}.
\end{aligned}
\end{equation}
Let us assume that $\mathcal{L}_i\left(x_i^{1}\right) - \mathcal{L}_i\left(x_i^*\right) \leq D, \forall i$, and we set $\eta_{min} = \sqrt{\frac{2D}{LMG^2T}}$. Then, we have
\begin{equation}
\frac{1}{T}\sum_{t=1}^{T} \mathbb{E}\left[\left\|\bar{g}_i^t\right\|^2\right] \leq 2\sqrt{\frac{LMG^2D}{2T}}.
\end{equation}
Thus, we can get
\begin{equation}
\frac{1}{mT}\sum_{i=1}^{m}\sum_{t=1}^{T} \mathbb{E}\left[\left\|\bar{g}_i^t\right\|^2\right] \leq 2\sqrt{\frac{LMG^2D}{2T}}.
\end{equation}
% Therefore, we finish the proof of Theorem \ref{theorem1}. 
% \hfill $\square$ 
\end{proof}

\subsection{Proof of Corollary \ref{corollary1}}
\begin{proof}
% \noindent \textit{Proof of Corollary \ref{corollary1}.} 
By rewriting inequality (\ref{leq1}), we have
% {\small
\begin{equation}
\label{leq32}
\begin{aligned}
&\mathbb{E}\left[\mathcal{L}_i\left(x_i^{t+1}\right)-\mathcal{L}_i\left(x_i^t\right)\right] \\ 
\leq &-\eta_{min } \mathbb{E}\left[\left\|\bar{g}_i^t\right\|^2\right] +\eta_{min }^2 \frac{L M G^2}{2}.
\end{aligned}
\end{equation}
% }
By the PL Assumption \ref{pl_assumption}, we have
{\small
\begin{equation}
\label{leq33}
\begin{aligned}
&\mathbb{E}\left[\mathcal{L}_i\left(x_i^{t+1}\right)-\mathcal{L}_i\left(x_i^t\right)\right] \\
\leq &-2\eta_{min}\mu \mathbb{E}\left[\mathcal{L}_i\left(x_i^{t}\right) - \mathcal{L}_i\left(x_i^*\right)\right] + \eta_{\min }^2 \frac{L M G^2}{2}.
\end{aligned}
\end{equation}
}Then, 
% {\small
\begin{equation*} 
\begin{aligned}
&\mathbb{E}\left[\mathcal{L}_i\left(x_i^{t+1}\right) - \mathcal{L}_i\left(x_i^*\right)\right]\\
= & \mathbb{E}\left[\mathcal{L}_i\left(x_i^{t}\right)- \mathcal{L}_i\left(x_i^*\right)\right] + \mathbb{E}\left[\mathcal{L}_i\left(x_i^{t+1}\right)-\mathcal{L}_i\left(x_i^t\right)\right]\\
\leq & \mathbb{E}\left[\mathcal{L}_i\left(x_i^{t}\right)- \mathcal{L}_i\left(x_i^*\right)\right]-2\eta_{min}\mu \mathbb{E}\left[\mathcal{L}_i\left(x_i^{t}\right) - \mathcal{L}_i\left(x_i^*\right)\right]\\
& + \eta_{min}^2\frac{LMG^2}{2}\\
= & (1-2\eta_{min}\mu) \mathbb{E}\left[\mathcal{L}_i\left(x_i^{t}\right) - \mathcal{L}_i\left(x_i^*\right)\right] + \eta_{min}^2\frac{LMG^2}{2}\\
\leq & \left(1-2 \eta_{min } \mu\right)^2 \mathbb{E}\left[\mathcal{L}_i\left(x_i^{t-1}\right)-\mathcal{L}_i\left(x_i^*\right)\right]\\
& +\sum_{\tau=0}^1\left(1-2 \eta_{\min } \mu\right)^\tau \eta_{\min }^2 \frac{L M G^2}{2} \\
\cdots \\
\leq & \left(1-2 \eta_{min } \mu\right)^{t+1} \mathbb{E}\left[\mathcal{L}_i\left(x_i^0\right)-\mathcal{L}_i\left(x_i^*\right)\right]\\
&+\sum_{\tau=0}^t\left(1-2 \eta_{min } \mu\right)^\tau \eta_{min }^2 \frac{L M G^2}{2}\\
\leq & (1-\eta_{min}\mu)^{t+1} D + \eta_{min}\frac{LMG^2}{4\mu}.
\end{aligned}
\end{equation*}
% }
Therefore, we have
% {\small
\begin{equation}
\begin{aligned}
&\mathbb{E}\left[\frac{1}{m}\sum_{i=1}^m\mathcal{L}_i\left(x_i^{t+1}\right)\right] - \mathcal{L}^* \\
\leq &(1-\eta_{min}\mu)^{t+1} D+ \eta_{min}\frac{LMG^2}{4\mu}.
\end{aligned}
\end{equation}
% }
If we set $\eta_{min} \leq \frac{\mu\epsilon}{LMG^2}$, after $O(\frac{1}{\epsilon} \log (\frac{1}{\epsilon}))$ steps, we have 
\begin{equation}
\mathbb{E}\left[\frac{1}{m}\sum_{i=1}^m\mathcal{L}_i\left(x_i^{t+1}\right)\right] - \mathcal{L}^* \leq \epsilon.
\end{equation}
% Thus, we finish the proof of Corollary \ref{corollary1}.
\end{proof}
% \hfill $\square$

\section{Generalization Bound} \label{sec:generalization_bound}
We further provide the generalization bound for HyperFL by employing the methodology outlined in \cite{baxter2000model}. First, we make the following assumption: 
\begin{assumption}
\label{assumption_gen_bound}
We assume the weights of hypernetworks $\varphi_i$, the client embeddings $\mathbf{v}_i$ and the weights of classifiers $\phi_i$ are bounded in a ball of radius $R$, in which the following Lipschitz conditions hold:
% 1. $\left| \mathcal{L}_i(h(\varphi_i; \mathbf{v}_i), \phi_i)  - \mathcal{L}_i(h(\varphi_i; \mathbf{v}_i), \phi^*_i) \right| \leq L_\phi||\phi_i - \phi^*_i||$\\
% 2. $\left| \mathcal{L}_i(h(\varphi_i; \mathbf{v}_i), \phi_i)  - \mathcal{L}_i(h^*(\varphi_i; \mathbf{v}_i), \phi_i) \right| \leq L_h||h_i - h^*_i||$\\
% 3. $\left|| h(\varphi_i^*; \mathbf{v}_i)  - h(\varphi_i; \mathbf{v}_i)) \right|| \leq L_\varphi||\varphi_i - \varphi^*_i||$\\
% 4. $\left|| h(\varphi_i; \mathbf{v}_i^*)  - h(\varphi_i; \mathbf{v}_i)) \right|| \leq L_\mathbf{v}||\mathbf{v}_i - \mathbf{v}^*_i||$
\begin{equation*}
\begin{aligned}
    &\left| \mathcal{L}_i(h(\varphi_i; \mathbf{v}_i), \phi_i)  - \mathcal{L}_i(h(\varphi_i; \mathbf{v}_i), \phi^*_i) \right| \leq L_\phi||\phi_i - \phi^*_i||, \\
    &\left| \mathcal{L}_i(h(\varphi_i; \mathbf{v}_i), \phi_i)  - \mathcal{L}_i(h^*(\varphi_i; \mathbf{v}_i), \phi_i) \right| \leq L_h||h_i - h^*_i||, \\
    &\left|| h(\varphi_i^*; \mathbf{v}_i)  - h(\varphi_i; \mathbf{v}_i)) \right|| \leq L_\varphi||\varphi_i - \varphi^*_i||, \\
    &\left|| h(\varphi_i; \mathbf{v}_i^*)  - h(\varphi_i; \mathbf{v}_i)) \right|| \leq L_\mathbf{v}||\mathbf{v}_i - \mathbf{v}^*_i||. 
\end{aligned}
\end{equation*}
\end{assumption}

\begin{theorem}
\label{theorem2}
Suppose we select $m$ clients at each communication round. Let the hypernetwork parameter space be of dimension $H$, the embedding space be of dimension $d$ and the classifier parameter space be of dimension $K$. Let the ${\phi}$, $\mathbf{v}$, ${\varphi}$ be the parameters learned from the individual dataset of clients. When Assumption \ref{assumption_gen_bound} holds, there exists 
{\small
\begin{equation} \label{eq:sample_size}
\begin{aligned}
S = &\mathcal{O}\bigg(\frac{d+H+K}{\epsilon^2}\log\Big(\frac{R(L_hL_\varphi+L_h L_v+ L_\phi)}{\epsilon}\Big)\\
&+\frac{1}{m\epsilon^2}\log\frac{1}{\delta}\bigg),
\end{aligned}
\end{equation}
}such that if the number of samples per client is greater than $S$, then we have with probability at least $1 - \delta$ for all ${\phi}$, $\mathbf{v}$, ${\varphi}$, 
{\small
\begin{equation} 
| \sum_i^m \frac{n_i}{N}(\mathcal{L}_i(h(\varphi^*_i; \mathbf{v}_i^*), \phi^*_i) - \mathcal{L}_i(h(\varphi_i; \mathbf{v}_i), \phi_i))| \leq \epsilon, 
\end{equation}
}where ${\phi^*}$, $\mathbf{v^*}$, ${\varphi^*}$ are the optimal parameters corresponding to the distribution of each individual client, respectively. 
\end{theorem}

% \subsection{Proof of Theorem \ref{theorem2}} \label{generalization_proof}

\begin{proof}    
% \textit{Proof.} 
First, we define the distance between ($\mathbf{\phi}$, $\mathbf{v}$, $\mathbf{\varphi}$) and ($\mathbf{\phi^*}$, $\mathbf{v^*}$, $\mathbf{\varphi^*}$) as 
\begin{equation}
\begin{aligned}
&d\big((\mathbf{\phi}, \mathbf{v}, \mathbf{\varphi}),(\mathbf{\phi^*}, \mathbf{v^*}, \mathbf{\varphi^*})\big) \\
= &\Big| \sum_i^m \frac{n_i}{N}(\mathcal{L}_i(h(\varphi^*_i; \mathbf{v}_i^*), \phi^*_i) - \mathcal{L}_i(h(\varphi_i; \mathbf{v}_i), \phi_i))\Big|,
\end{aligned}
\end{equation}
where $N = \sum_{i=1}^m n_i$.
By the Theorem 4 from \cite{baxter2000model}, we can find an $\epsilon$-covering in $d((\mathbf{\phi}, \mathbf{v}, \mathbf{\varphi}),(\mathbf{\phi^*}, \mathbf{v^*}, \mathbf{\varphi^*}))$. Then, according to the notations used in our paper, we have that $S = \mathcal{O}(\frac{1}{m\epsilon^2}log(\frac{\mathcal{C}(\epsilon, \mathcal{H}_l^n)}{\delta}))$, where $\mathcal{C}(\epsilon, \mathcal{H}_l^n)$ is the covering number of $\mathcal{H}_l^n$. In our case, each element of $\mathcal{H}_l^n$ is parameterized by $\mathbf{\phi}, \mathbf{v}, \mathbf{\varphi}$. Therefore, from the triangle inequality and the Lipschitz conditions in Assumption \ref{assumption_gen_bound}, we can get 
{\small
\begin{equation}
\begin{aligned}
&d\big((\mathbf{\phi}, \mathbf{v}, \mathbf{\varphi}),(\mathbf{\phi^*}, \mathbf{v^*}, \mathbf{\varphi^*})\big) \\
= &\big|\sum_i^m\frac{n_i}{N}(\mathcal{L}_i(h(\varphi^*_i; \mathbf{v}_i^*), \phi^*_i) -  \mathcal{L}_i(h(\varphi_i; \mathbf{v}_i), \phi_i))\big| \\
= &\big|\sum_i^m\frac{n_i}{N}(\mathcal{L}_i(h(\varphi^*_i; \mathbf{v}_i^*), \phi^*_i) -  \mathcal{L}_i(h(\varphi_i; \mathbf{v}_i^*), \phi_i^*) \\
&+ \mathcal{L}_i(h(\varphi_i; \mathbf{v}_i^*), \phi_i^*) -  \mathcal{L}_i(h(\varphi_i; \mathbf{v}_i), \phi_i^*) + \mathcal{L}_i(h(\varphi_i; \mathbf{v}_i), \phi_i^*) \\
&-  \mathcal{L}_i(h(\varphi_i; \mathbf{v}_i), \phi_i))\big| \\
\leq &\sum_i^m \frac{n_i}{N}\big| (\mathcal{L}_i(h(\varphi^*_i; \mathbf{v}_i^*), \phi^*_i) - \mathcal{L}_i(h(\varphi_i; \mathbf{v}_i^*), \phi_i^*) \\ 
&+  \mathcal{L}_i(h(\varphi_i; \mathbf{v}_i^*), \phi_i^*) -  \mathcal{L}_i(h(\varphi_i; \mathbf{v}_i), \phi_i^*) \\
&+  \mathcal{L}_i(h(\varphi_i; \mathbf{v}_i), \phi_i^*) -  \mathcal{L}_i(h(\varphi_i; \mathbf{v}_i), \phi_i))\big| \\
\leq &\sum_{i=1}^m\frac{n_i}{N} (L_h||h(\varphi_i^*; \mathbf{v}_i^*)-h(\varphi_i; \mathbf{v}_i^*)||+L_h||h(\varphi_i; \mathbf{v}_i^*) \\
&- h(\varphi_i; \mathbf{v}_i)||+L_\phi||\phi_i^* - \phi_i||) \\
\leq & L_h L_\varphi||\varphi_i^*-\varphi_i|| + L_h L_v||\mathbf{v}_i^* - \mathbf{v}_i||+L_\phi||\phi_i^* - \phi_i||.
\end{aligned}
\end{equation}
}

\noindent Now if there is a parameter space such that $\phi_i$, $\mathbf{v}_i$ and $\varphi_i$ have corresponding optimal point $\phi_i^*$, $\mathbf{v}_i^*$ and $\varphi_i^*$, which are $\frac{\epsilon}{L_h L_\varphi+ L_h L_v+L_\phi}$ away, respectively, we can get an upper bound of the distance between our model and optimal model, which is an $\epsilon$-covering in $d((\mathbf{\phi}, \mathbf{v}, \mathbf{\varphi}),(\mathbf{\phi^*}, \mathbf{v^*}, \mathbf{\varphi^*}))$ matrix. From here we see that $\log(\mathcal{C}(\epsilon, \mathcal{H}_l^n)) = \mathcal{O}(m(d+H+K)\log(\frac{RL_h(L_\varphi + L_v)+R L_\phi}{\epsilon}))$.  
% \hfill $\square$
\end{proof}

Theorem \ref{theorem2} suggests that $S$ is influenced by several factors: the dimension of the parameters space, the number of clients, and the values of the Lipschitz constants. Specifically, the first part of right hand side of Eq. (\ref{eq:sample_size}) is determined by the dimensions of the embedding vectors, hypernetwork parameters, and classifier parameters. This component is independent of the number of clients $m$, as each client has its unique embedding vector, hypernetwork and classifier.
% , which are not shared between clients. %这个地方想说的是因为每个client都有自己的hypernetwork，classfier，embedding vector，所以定理中的式子13的第一项与clients 的数量m无关
Additionally, this theorem points out that generalization depends on the Lipschitz constants, which influence the effective space reachable by the personalized models of clients. This indicates a trade-off between the generalization ability and the flexibility of the personalized model. 
% The detailed proofs for Theorem  \ref{theorem2} are provided in Appendix. 
%这个地方想说的是assumption 中 Lipshcitz constants会影响到我们的泛化能力。因为lipshcitz constants越大我们的S就会越大，导致generalization能力下降。然而如果限制lipchscitz constants变小的话就会影响我们每个clients的personalized model的flexibility。所以这是一个tradeoff。
% \ref{generalization_proof}.

\section{Related Work}
\label{app_sec:related_work}

\paragraph{Gradient Inversion Attacks.} 
Gradient Inversion Attacks (GIA) \cite{fredrikson2015model, zhu2019deep} is a class of adversarial attacks that exploit the gradients of a machine learning model to infer sensitive information about the training data by leveraging the fact that gradients contain information about the relationship between the input and the model's output. The basic idea behind GIA is to intentionally modify the input data in a way that maximizes the magnitude of the gradients with respect to the sensitive information of interest. By iteratively adjusting the input data based on the gradients, an attacker can gradually approximate the sensitive information, such as private attributes or training data samples, that the model was trained on \cite{phong2017privacy, zhu2019deep, geiping2020inverting, yin2021see, luo2022effective, geng2023improved, kariyappa2023cocktail}. GIA can pose a significant threat to privacy in scenarios where the model is used in sensitive applications or when the model's training data contains sensitive information. These attacks highlight the need for robust privacy protection mechanisms to mitigate the risk of information leakage through gradients.

\paragraph{Privacy Protection in Federated Learning.} 
Although the local data are not exposed in FL, the exchanged model gradients may still leak sensitive information about the data that can be leveraged by GIA to recover them \cite{geiping2020inverting, huang2021evaluating, li2022e2egi, hatamizadeh2023gradient, kariyappa2023cocktail}. To further protect the data privacy, additional defense methods have been integrated into FL, and can be categorized into three classes: Secure Multi-party Computing (SMC) \cite{yao1982protocols} based methods \cite{bonawitz2017practical, mugunthan2019smpai, mou2021verifiable, xu2022non}, Homomorphic Encryption (HE) \cite{gentry2009fully} based methods \cite{zhang2020batchcrypt, zhang2020privacy, ma2022privacy, park2022privacy} and Differential Privacy (DP) \cite{dwork2006differential} based methods \cite{geyer2017differentially, mcmahan2018learning, yu2020salvaging, bietti2022personalization, shen2023share}. 
SMC, originating from Yao's Millionaire problem \cite{yao1982protocols}, is a framework that aims to protect the input data of each participating party by employing encryption techniques during collaborative computations. With the development of FL, SMC techniques have evolved and been adapted to federated systems to enhance the protection of sensitive data through parameter encryption \cite{bonawitz2017practical, mugunthan2019smpai, mou2021verifiable, xu2022non}. HE, introduced by Gentry \cite{gentry2009fully}, is an encryption algorithm that preserves the homomorphic property of ciphertexts. In the context of FL, HE enables the central server to perform algebraic operations directly on encrypted parameters without the need for decryption \cite{zhang2020batchcrypt, zhang2020privacy, ma2022privacy, park2022privacy}. DP is a widely adopted privacy-preserving technique in both industry and academia by clipping the gradients and adding noise to personal sensitive attribute \cite{dwork2006differential, abadi2016deep}. In the context of FL, DP is employed to prevent inverse data retrieval by clipping gradients and adding noise to participants' uploaded parameters \cite{geyer2017differentially, mcmahan2018learning, yu2020salvaging, bietti2022personalization, shen2023share}. 
However, SMC and HE methods are unsuitable to DNN models due to their extremely high computation and communication cost, while DP methods usually introduce additional computation cost and result in a decrease in model performance \cite{bonawitz2017practical, zhang2020privacy, geyer2017differentially, mcmahan2018learning}. 
Apart from these defense methods with theoretical guarantees, there are other empirical yet effective defense strategies, such as gradient pruning/masking \cite{zhu2019deep, huang2021evaluating, li2022auditing}, noise addition \cite{zhu2019deep, wei2020framework, huang2021evaluating, li2022auditing}, Soteria \cite{sun2020provable}, PRECODE \cite{scheliga2022precode}, and FedKL \cite{ren2023gradient}. However, these methods still suffer from privacy-utility trade-off problems, as shown in Tables 2 and 5 in \cite{huang2021evaluating} for gradient pruning/masking, Tables 1, 2, and 3 in PRECODE \cite{scheliga2022precode}, Table 1 in FedKL \cite{ren2023gradient}, and Figure 5 in Soteria \cite{sun2020provable}. In contrast to these approaches, with the help of hypernetworks \cite{david2017hypernetworks}, this work proposes a novel FL framework that effectively “breaks the direct connection” between the shared parameters and the local private data to defend against GIA while achieving a favorable privacy-utility trade-off.

\paragraph{Hypernetworks in Federated Learning.} Hypernetworks \cite{david2017hypernetworks} are deep neural networks that generate the weights for another network, known as the target network, based on varying inputs to the hypernetwork. Recently, there have been some works that incorporate hypernetworks into FL for learning personalized models \cite{shamsian2021personalized, carey2022robust, li2023fedtp, tashakori2023semipfl, lin2023federated}. All of these methods adopt the similar idea that a central hypernetwork model are trained on the server to generate a set of models, one model for each client, which aims to generate personalized model for each client. 
Since the hypernetwork and client embeddings are trained on the server, which makes the server possessing all the information about the local models, enabling the server to recover the original inputs by GIA ((see Table \ref{tab:attack} and Figure \ref{fig:attack} in Appendix)).
In contrast to existing approaches, this work presents a Hypernetwork Federated Learning (HyperFL) framework, which prioritizes data privacy preservation over personalized model generation through the utilization of hypernetworks.

\section{Additional Experimental Results and Experimental Details.}

\subsection{Details of Experimental Setup} \label{sec:exp_details}

\paragraph{Datasets.} 
For the Main Configuration HyperFL, we evaluate our method on four widely-used image classification datasets: (1) EMNIST (Extended MNIST) \cite{cohen2017emnist}, a dataset with 62 categories of handwritten characters, including 10 digits, 26 uppercase letters, and 26 lowercase letters; (2) Fashion-MNIST \cite{xiao2017fashion}, a dataset designed for fashion product images, containing 10 categories of clothing items; (3) CIFAR-10 \cite{krizhevsky2009learning}, a widely used benchmark dataset for image classification tasks, consisting of 60,000 color images distributed across 10 different classes; and (4) CINIC-10 \cite{darlow2018cinic}, a composite image dataset that combines samples from CIFAR-10 and ImageNet \cite{russakovsky2015imagenet}, comprising 270,000 images spanning 10 different classes.

Similar to \cite{karimireddy2020scaffold, zhang2021personalized, xu2023personalized}, we create a non-IID data distribution by ensuring all clients have the same data size, in which $s\%$ of data ($20\%$ by default) are uniformly sampled from all classes and the remaining $(100 - s)\%$ from a set of dominant classes for each client. Following \cite{xu2023personalized}, we evenly divide all clients into multiple groups, with each group having the same dominant classes. Specifically, for the 10-category Fashion-MNIST, CIFAR-10 and CINIC-10 datasets, we divide clients into 5 groups. Each group is assigned three consecutive classes as the dominant class set, starting from class 0, 2, 4, 6, and 8 for the respective groups. For EMNIST dataset, we divide clients into 3 groups, with each group assigned the dominant set of digits, uppercase letters, and lowercase letters, respectively.

For the HyperFL-LPM, we evaluate our method on the EMNIST \cite{cohen2017emnist} and CIFAR-10 \cite{krizhevsky2009learning} datasets.

\paragraph{Model Architectures.} 
For the Main Configuration HyperFL, simlar to \cite{xu2023personalized}, we adopt two different CNN target models for EMNIST/Fashion-MNIST and CIFAR-10/CINIC-10, respectively. The first CNN target model is built with two convolutional layers. The first CNN target model is built with two convolutional layers (16 and 32 channels) followed by max pooling layers, two fully-connected layers (128 and 10 units), and a softmax output layer, using LeakyReLU activation functions \cite{xu2015empirical}. The second CNN model is similar to the first one but adds one more 64-channel convolution layer. The hypernetwork is a fully-connected neural network with one hidden layer, multiple linear heads per target weight tensor. The client embeddings are learnable vectors with dimension equals 64. %The first layer has 16 channels, while the second layer has 32 channels. Each convolutional layer is followed by a max pooling layer. The model also includes two fully-connected layers with 128 and 10 units, respectively, before the softmax output layer. For activation functions, the LeakyReLU function \cite{xu2015empirical} is utilized. The second target CNN model is similar to the first one but has one more convolution layer with 64 channels. The hypernetwork is a simple fully-connected neural network, with one hidden layers and multiple linear heads per target weight tensor and the client embedding is a learnable vector with dimension equals 64.

For the HyperFL-LPM, we adopt the ViT-S/16 \cite{dosovitskiy2021image} and ResNet-18 \cite{he2016deep} pre-trained on the ImageNet dataset \cite{deng2009imagenet} as the feature extractor. When the pre-trainde model is ResNet, the adapter is inserted behind each resnet block. The adapter within each transformer block consists of a down-projection layer, ReLU activation functions \cite{nair2010rectified}, and a up-projection layer. The hypernetwork is a fully-connected neural network with one hidden layer, multiple linear heads per target weight tensor. The client embeddings are learnable vectors with dimension equals 64. 

\begin{figure*}[h]
% \begin{wrapfigure}{r}{0.55\textwidth}
% \vskip -0.2in
    \centering
\includegraphics[width=1\textwidth]{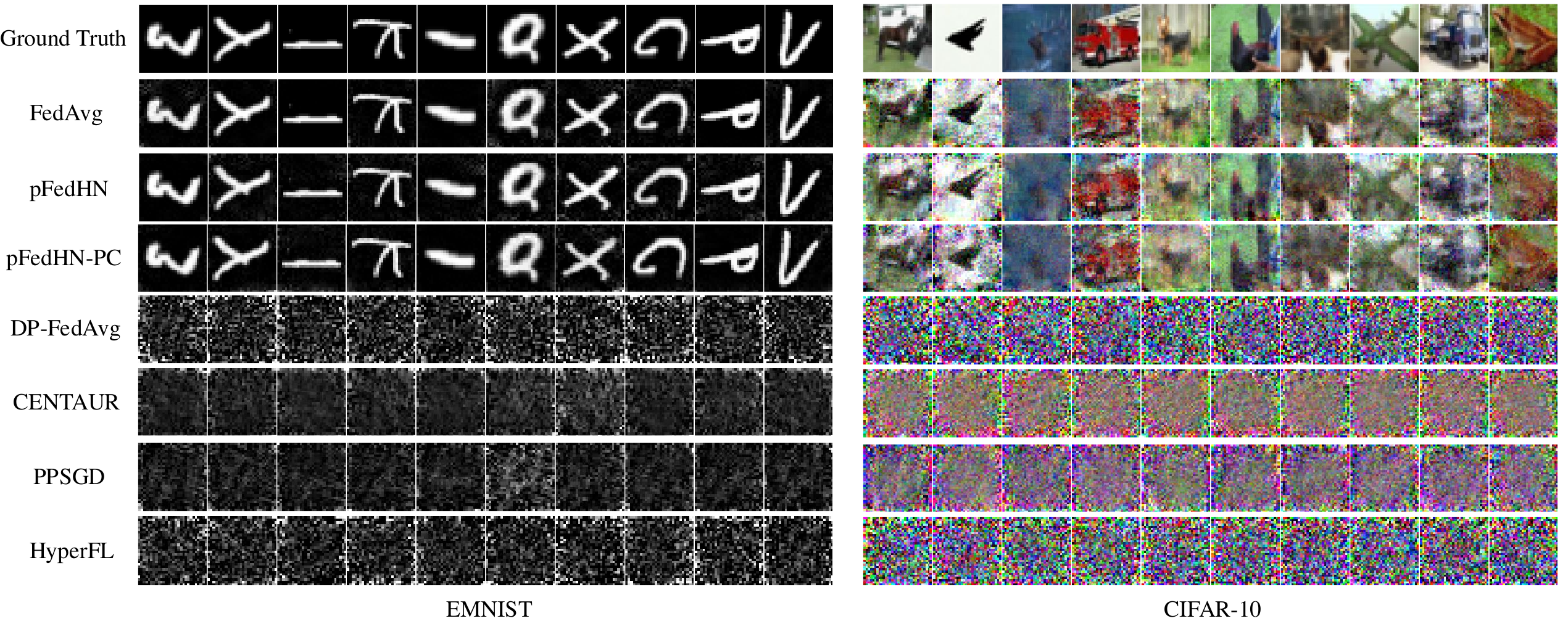}
% \vskip -0.05in
    \caption{Reconstructed images of IG.}
    \label{fig:attack}
\end{figure*}
% \end{wrapfigure}

\paragraph{Compared Methods.}  
For the Main Configuration HyperFL, we compare the proposed method with the following approaches: (1) Local-only, where clients train models locally without collaboration; (2) FedAvg \cite{mcmahan2017communication}, a widely-used FL method; (3) pFedHN \cite{shamsian2021personalized}, that utilizes a central hypernetwork model trained on the server to generate a set of models, one model for each client; and some DP-based FL methods, including (4) DP-FedAvg \cite{mcmahan2018learning}, which incorporates differential privacy into FedAvg; (5) PPSGD \cite{bietti2022personalization}, a personalized private SGD algorithm with user-level differential privacy; and (6) CENTAUR \cite{shen2023share}, which trains a single differentially private global representation extractor while allowing personalized classifier heads. However, all these compared DP-based FL methods (i.e., DP-FedAvg, PPSGD, and CENTAUR) focus on user-level DP setting \cite{mcmahan2018learning}, which cannot guarantee protection against \textit{honest-but-curious} server attacks as they upload original gradients to the server. Therefore, we adapt these methods to fit the ISRL-DP setting \cite{lowy2023private}, where users trust their own client but not the server or other clients, thereby defending against \textit{honest-but-curious} server attacks.

For the HyperFL-LPM, we compare our method with (1) Local-only with fixed feature extractor; (2) Local-only with adapter fine-tuning; (3) FedAvg with fixed feature extractor; and (4) FedAvg with adapter fine-tuning.

\paragraph{Training Settings.} 
For the Main Configuration HyperFL, mini-batch SGD \cite{ruder2016overview} is adopted as the local optimizer for all approaches. Similar to \cite{xu2023personalized}, we set the step sizes $\eta_h$ and $\eta_v$ for local training to 0.01 for EMNIST/Fashion-MNIST and 0.02 for CIFAR-10/CINIC-10. The setp size $\eta_g$ is set to 0.1 for all the datasets. The weight decay is set to 5e-4 and the momentum is set to 0.5. The batch size is fixed to B = 50 for all datasets except EMNIST (B = 100). The client embedding dimension is set to 64. The number of local training epochs is set to 5 for all FL approaches and the number of global communication rounds is set to 200 for all datasets. Furthermore, following \cite{xu2023personalized}, we conduct experiments on two setups, where the number of clients is 20 and 100, respectively. For the latter, we apply random client selection with sampling rate 0.3 along with full participation in the last round. The training data size per client is set to 600 for all datasets except EMNIST, where the size is 1000. For the DP-based FL methods, the DP budget $\epsilon$ is set to 4 and the Gaussian noise $\sigma$ is $1e-5$ to satisfy the ($\epsilon$, $\sigma$) privacy guarantee. Average test accuracy of all local models is reported for performance evaluation. 

For the HyperFL-LPM, we conducted experiments with 20 clients. Furthermore, differently from HyperFL, batch size 16 is adopted for all datasets. The step sizes $\eta_h$ and $\eta_v$ for local training are 0.02 for EMNIST and 0.1 for CIFAR-10 when using ViT pre-trained models. When using ResNet pre-trained models, the step size is set to 0.01 for all datasets.

\paragraph{Privacy Evaluation.} 
EMNIST and CIFAR-10 are used to evaluate privacy preservation capability of the proposed HyperFL. We choose a subset of 50 images from each dataset to evaluate the privacy leakage. A batch size of one is used.
For experimental comparison, we set all the unknown variable in HyperFL are learnable and optimized simultaneously for IG \cite{geiping2020inverting} and ROG \cite{yue2023gradient}.
% and modify the objective (\ref{eq:att_ours_bilevel}) to:
% \begin{equation} \label{eq:recover_exp}
%     \arg \min _{x, z, \mathbf{v}, \theta} \mathcal{L}_{\text {grad }}\left(x, z, \theta, \mathbf{v}; \varphi, \nabla_\varphi \mathcal{L}_{\varphi, \phi}\left(x^*, \mathbf{v}^*, y^*\right)\right),
% \end{equation}
% where $z=f(x)$ denotes the hidden representation. 
For the optimization of IG \cite{geiping2020inverting}, we optimize the attack for 10,000 iterations using the Adam optimizer \cite{kingma2015adam}, with an initial learning rate of 0.1. The learning rate is decayed by a factor of 0.1 at 3/8, 5/8, and 7/8 of the optimization process. The coefficient of the TV regularization term is set to 1e-6.
For the optimization of ROG \cite{yue2023gradient}, the Adam optimizer \cite{kingma2015adam} with a learning rate of 0.05 is adopted, and the total number of iterations is set to 100.
% The Peak Signal-to-Noise Ratio (PSNR) is calculated to measure the similarity between the original images and the images recovered from GIA. %PSNR is an objective standard for image evaluation, and it is defined as the logarithm of the ratio of the squared maximum value of RGB image fluctuation over MSE between two images. 
% Higher PSNR score indicates a greater similarity between two images. %That means the lower the PSNR score, the better the defense of the model. 
To further demonstrate the privacy preservation capability of the proposed HyperFL, we design a tailored attack method. Specifically, we first recover the client embedding according to Eq. (\ref{eq:recover_v}), and then recover the input data by solving the upper-level subproblem in objective Eq. (\ref{eq:att_ours_bilevel}). Since $\Delta_\theta$ cannot be obtained \footnote{The client embedding changes at each iteration, making $\Delta_\theta$ impossible to calculate.}, we utilize model inversion attack methods \cite{he2019model,he2020attacking,jiang2022comprehensive,erdougan2022unsplit} to solve the upper-level subproblem in objective Eq. (\ref{eq:att_ours_bilevel}).
Peak signal to noise ratio (PSNR) \cite{hore2010image}, structural similarity (SSIM) \cite{wang2004image}, and learned perceptual image patch similarity (LPIPS) \cite{zhang2018unreasonable} are adopted as the metrics for reconstruction attacks on image data. Lower LPIPS, higher PSNR and SSIM of reconstructed images indicate better attack performance.

All experiments are conducted on NVIDIA GeForce RTX 3090 GPUs.

\begin{figure*}[h]
% \begin{wrapfigure}{r}{0.55\textwidth}
% \vskip -0.2in
    \centering
\includegraphics[width=1\textwidth]{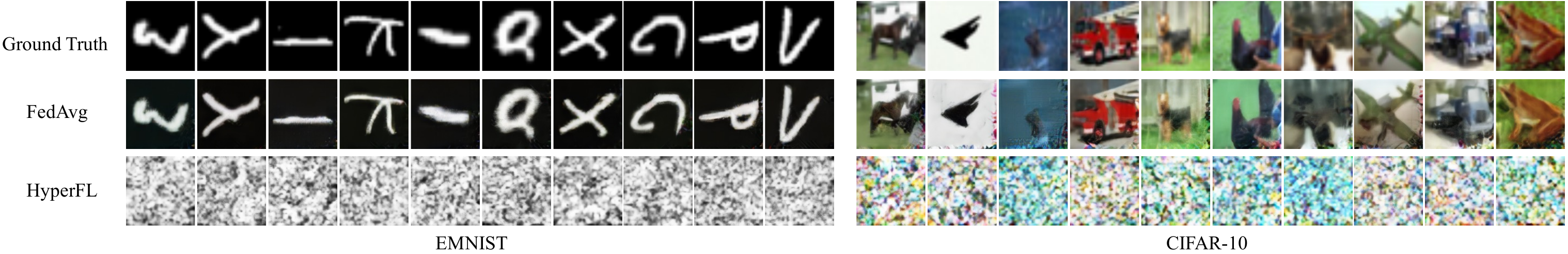}
% \vskip -0.05in
    \caption{Reconstructed images of ROG.}
    \label{fig:attack_rog}
\end{figure*}
% \end{wrapfigure}

\subsection{Additional Experimental Results}

\paragraph{Privacy Evaluation.} 

The visualization results of IG \cite{geiping2020inverting} of the first 10 images are provided in Figure \ref{fig:attack}. From this figure, we can observe that the native FedAvg and pFedHN methods have a much higher risk of leaking data information, as indicated by the reconstructed images closely resembling the original ones. Although introducing DP improves data privacy, there is a significant drop in model performance, as shown in Table \ref{tab:main_result}. In contrast, HyperFL achieves a similar level of privacy protection while outperforming all DP-based methods and the native FedAvg in terms of model accuracy.

The reconstructed and visualization results of ROG \cite{yue2023gradient} are provided in Table \ref{tab:attack_rog} and Figure \ref{fig:attack_rog}. From these results, we can observe that the proposed HyperFL can also defend against SOTA attack method.

\begin{table} [h]
    % \begin{wraptable} {r} {0.4\linewidth}
        % \vskip -0.2in
        \centering
          % \vskip 0.05in
        \resizebox{0.75\linewidth}{!}{
          \begin{tabular}{l ccc ccc}
            \toprule
            & \multicolumn{3}{c}{EMNIST} & \multicolumn{3}{c}{CIFAR-10} \\
            \cmidrule(lr){2-4} \cmidrule(lr){5-7} 
            Method & PSNR  & SSIM  & LPIPS  & PSNR  & SSIM  & LPIPS  \\
            \midrule
            FedAvg & 24.26 & 0.9516 & 0.3024 & 23.09 & 0.9228 & 0.4363  \\
            \midrule
            HyperFL & 3.44 & 0.0459 & 0.7883 & 7.78 & 0.0137 & 0.7802  \\
            \bottomrule
          \end{tabular}
        }
    \caption{Reconstruction results of ROG.}
\label{tab:attack_rog}
\vskip -0.1in
\end{table}
        % \end{wraptable}

The reconstructed visualization results of the tailored attack method are presented in Figure \ref{fig:attack_tailored}. These results demonstrate that even the tailored attack method is unable to recover any information from the proposed HyperFL framework, thereby showcasing the robust privacy preservation capability of HyperFL.

\begin{figure}[h]
% \begin{wrapfigure}{r}{0.55\textwidth}
% \vskip -0.2in
    \centering
\includegraphics[width=1\linewidth]{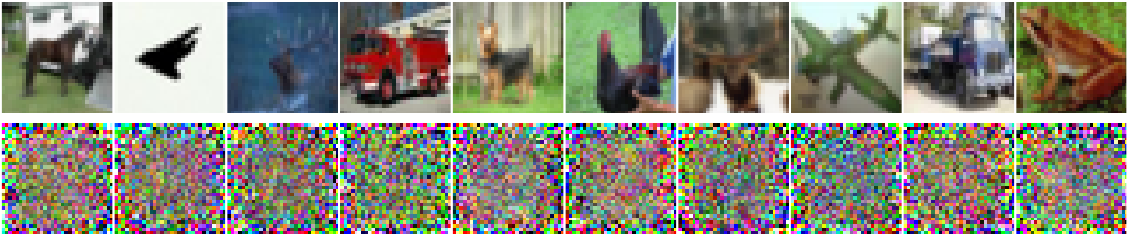}
% \vskip -0.05in
    \caption{Reconstructed images of the tailored attack method. The first row contains the original images, while the second row shows the reconstruction results.}
    \label{fig:attack_tailored}
\vskip -0.1in
\end{figure}
% \end{wrapfigure}

\paragraph{Learned Client Embeddings.} 
In our experiments, we learn to represent each client using a trainable embedding vector $\mathbf{v}_i$. These embedding vectors are randomly initialized to the same value. By setting these embedding vectors trainable, they can learn a continuous semantic representation over the set of clients. The t-SNE visualization \cite{van2008visualizing} of the learned client embeddings of the EMNIST dataset with 20 clients is shown in Figure \ref{fig:emb_rela_t_sne}. Form this figure we can see that there is a distinct grouping of the learned client embeddings into three clusters, which aligns with the data partitioning we employed, as shown in Figure \ref{fig:label_dist}. This phenomenon demonstrates the meaningfulness of the learned client embeddings in capturing the underlying relationship of the clients. In this way, personalized feature extractor parameters for each client can be generated by taking the meaningful client embedding as input for the hypernetwork. Therefore, the model can achieve better performance by adopting personalized feature extractors.

% \begin{figure}[h]
%     \centering
%     \includegraphics[width=0.85\linewidth]{embedding_relationship_t_sne.pdf}
%     \caption{The t-SNE visualization of the learned client embeddings of the EMNIST dataset with 20 clients.}
%     \label{fig:emb_rela_t_sne}
% \end{figure}

\begin{figure}[h]
  \centering
  \subfigure[]{\includegraphics[width=0.48\linewidth]{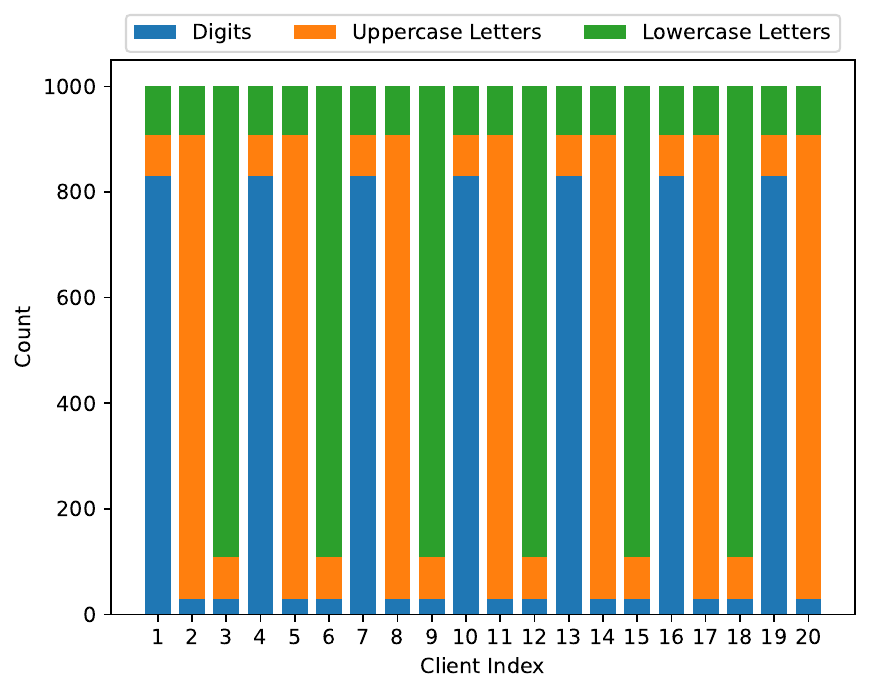} \label{fig:label_dist}}
  % \begin{subfigure}[h]{0.49\linewidth}
  %   \includegraphics[width=\linewidth]{label_dist.pdf}
  %   \caption{}
  %   \label{fig:label_dist}
  % \end{subfigure}
  % \vspace{-1cm}
  \hfill
  \subfigure[]{\raisebox{0.4cm}{\includegraphics[width=0.48\linewidth] {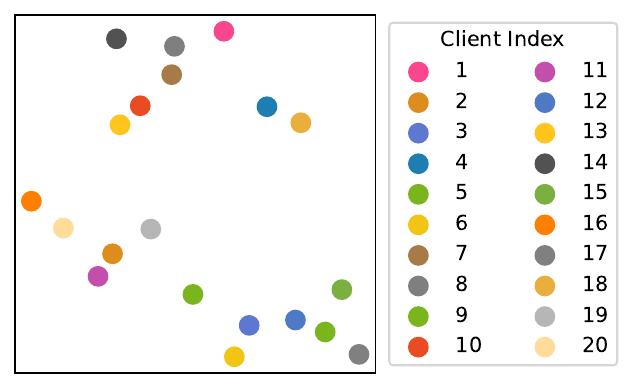}} \label{fig:emb_rela_t_sne}} 
  % \begin{subfigure}[h]{0.49\linewidth}
    % \vskip 0.13in
  %   \includegraphics[width=\linewidth]{embedding_relationship_t_sne.pdf}
  %   \vskip 0.15in
  %   \caption{}
  %   \label{fig:emb_rela_t_sne}
  % \end{subfigure}
  \vskip -0.1in
  \caption{(a) Label distribution of the EMNIST dataset with 20 clients. (b) The t-SNE visualization of the learned client embeddings of the EMNIST dataset with 20 clients.}
  \label{fig:simlarity}
\end{figure}

%%%%%%%%%%%%%%%%%%%%%%%%%%%%%%%%%%%%%%%%%%%%%%%%%%%%%%%%%%%%

\end{document}